\documentclass[11pt]{article}

\usepackage{macros}
\usepackage{fullpage}
\usepackage{amsfonts,graphicx}
\usepackage{amsmath,amssymb,amsthm}
\usepackage{graphics}
\usepackage{color}
\usepackage{enumerate}
\usepackage[numbers]{natbib}
\usepackage{floatrow}
\usepackage{algorithm}
\usepackage{enumitem}
\newfloatcommand{capbtabbox}{table}[][\FBwidth]
\usepackage[bookmarks=false,colorlinks,citecolor=blue,urlcolor=blue]{hyperref}

\newtheorem{theorem}{Theorem}
\newtheorem{lemma}{Lemma}

\begin{document}

\title{Convexified Convolutional Neural Networks}

\author{
Yuchen Zhang\footnote{\footnotesize Computer Science Department, Stanford University, Stanford, CA 94305. Email: {\tt zhangyuc@cs.stanford.edu}.}
\qquad
Percy Liang\footnote{\footnotesize Computer Science Department, Stanford University, Stanford, CA 94305. Email: {\tt pliang@cs.stanford.edu}.}
\qquad
Martin J. Wainwright\footnote{\footnotesize Department of Electrical Engineering and Computer Science and Department of Statistics, University of California Berkeley, Berkeley, CA 94720. Email: {\tt wainwrig@eecs.berkeley.edu}.}
}

\maketitle

\begin{abstract}
We describe the class of convexified convolutional neural networks
(CCNNs), which capture the parameter sharing of convolutional neural
networks in a convex manner. By representing the nonlinear
convolutional filters as vectors in a reproducing
kernel Hilbert space, the CNN parameters can be represented as a low-rank
matrix, which can be relaxed to obtain a convex optimization
problem.  For learning two-layer convolutional neural networks, we
prove that the generalization error obtained by a convexified CNN
converges to that of the best possible CNN.  For learning deeper
networks, we train CCNNs in a layer-wise manner. Empirically, CCNNs
achieve performance competitive with CNNs trained by backpropagation, SVMs,
fully-connected neural networks, stacked denoising auto-encoders, and
other baseline methods.
\end{abstract}

\renewcommand*{\thefootnote}{\arabic{footnote}}
\setcounter{footnote}{0}


\section{Introduction}

Convolutional neural networks (CNNs)~\cite{lecun1998gradient} have
proven successful across many tasks in machine learning and artificial
intelligence, including image
classification~\cite{lecun1998gradient,krizhevsky2012imagenet}, face
recognition~\cite{lawrence1997face}, speech
recognition~\cite{hinton2012deep}, text
classification~\cite{wang2012end}, and game
playing~\cite{mnih2015human,silver2016mastering}.  There are two principal advantages
of a CNN over a fully-connected neural network:
(i) sparsity---that each nonlinear convolutional filter acts only on a local patch of the input,
and (ii) parameter sharing---that the same filter is applied to each patch.

However, as with most neural networks, the standard approach to
training CNNs is based on solving a nonconvex optimization problem
that is known to be NP-hard~\cite{blum1992training}.  In practice,
researchers use some flavor of stochastic gradient method, in which
gradients are computed via backpropagation~\cite{bottou1998online}.
This approach has two drawbacks: (i) the rate of convergence, which is
at best only to a local optimum, can be slow due to nonconvexity (for
instance, see the paper~\cite{fahlman1988empirical}), and (ii) its
statistical properties are very difficult to understand, as the actual
performance is determined by some combination of the CNN architecture
along with the optimization algorithm.



In this paper, with the goal of addressing these two challenges, we
propose a new model class known as \emph{convexified convolutional
  neural networks} (CCNNs).  These models have two desirable features.
First, training a CCNN corresponds to a convex optimization problem,
which can be solved efficiently and optimally via a projected
gradient algorithm.  Second, the statistical properties of CCNN models
can be studied in a precise and rigorous manner.  We obtain CCNNs by
convexifying two-layer CNNs; doing so requires overcoming two
challenges. First, the activation function of a CNN is nonlinear.  In
order to address this issue, we relax the class of CNN filters to a
reproducing kernel Hilbert space (RKHS).  This approach is inspired by
our earlier work~\cite{zhang2015ell_1}, involving a subset of the
current authors, in which we developed this relaxation step for
fully-connected neural networks.  Second, the parameter sharing
induced by CNNs is crucial to its effectiveness and must be preserved.
We show that CNNs with RKHS filters can be parametrized by a low-rank
matrix.  Further relaxing the low-rank constraint to a nuclear norm
constraint leads to our final formulation of CCNNs.

On the theoretical front, we prove an \emph{oracle inequality} on
generalization error achieved by our class of CCNNs, showing that it
is upper bounded by the best possible performance achievable by a
two-layer CNN given infinite data---a quantity to which we refer as
the oracle risk---plus a model complexity term that decays to zero
polynomially in the sample size.  Our results show that the sample
complexity for CCNNs is significantly lower than that of the
convexified fully-connected neural network~\citep{zhang2015ell_1},
highlighting the importance of parameter sharing.  For models with
more than one hidden layer, our theory does not apply, but we provide
encouraging empirical results using a greedy layer-wise training
heuristic. We then apply CCNNs to the MNIST handwritten digit dataset
as well as four variation datasets~\cite{WinNT}, and find that it
achieves the state-of-the-art performance. On the CIFAR-10 dataset,
CCNNs outperform CNNs of the same depths, as well as other
baseline methods that do not involve nonconvex optimization. We also
demonstrate that building CCNNs on top of existing CNN filters
improves the performance of CNNs.

The remainder of this paper is organized as follows.  We begin in
Section~\ref{sec:setup} by introducing convolutional neural networks,
and setting up the empirical risk minimization problem studied in this
paper.  In Section~\ref{sec:single-layer-algorithm}, we describe the
algorithm for learning two-layer CCNNs, beginning with the
simple case of convexifying CNNs with a linear activation function,
then proceeding to convexify CNNs with a nonlinear activation.  We
show that the generalization error of a CCNN converges to that of the
best possible CNN. In Section~\ref{sec:multi}, we describe several
extensions to the basic CCNN algorithm, including averaging pooling,
multi-channel input processing, and the layer-wise learning of
multi-layer CNNs. In Section~\ref{sec:experiment}, we report the
empirical evaluations of CCNNs. We survey related work in
Section~\ref{sec:related-work} and conclude the paper in
Section~\ref{sec:conclusion}.

\myparagraph{Notation} For any positive integer $n$, we use $[n]$ as a
shorthand for the discrete set $\{1,2,\dots, n\}$.  For a rectangular
matrix $A$, let $\lncs{A}$ be its nuclear norm, $\ltwos{A}$ be its
spectral norm (i.e., maximal singular value), and $\lfs{A}$ be its
Frobenius norm. We use $\ell^2(\N)$ to denote the set of countable dimensional
vectors $v = (v_1,v_2,\dots)$ such that $\sum_{\ell=1}^\infty v_\ell^2 < \infty$. For any vectors $u, v\in \ell^2(\N)$, the inner product $\inprod{u}{v} \defeq \sum_{\ell=1}^\infty u_iv_i$ and the $\ell_2$-norm $\ltwos{u} \defeq \sqrt{\inprod{u}{u}}$ are well defined.


\section{Background and problem set-up}
\label{sec:setup}

In this section, we formalize the class of convolutional neural
networks to be learned and describe the associated nonconvex
optimization problem.


\subsection{Convolutional neural networks.}
\label{sec:setup-cnn}

At a high level, a two-layer CNN\footnote{Average pooling and multiple
  channels are also an integral part of CNNs, but these do not present any new
  technical challenges, so that we defer these extensions to
  Section~\ref{sec:multi}.}  is a particular type of function that
maps an input vector $x \in \R^{d_0}$ (e.g., an image) to an output
vector in $y \in \R^{d_2}$ (e.g., classification scores for the $d_2$
classes).  This mapping is formed in the following manner:
\begin{itemize}[leftmargin = *]
\item First, we extract a collection of $P$ vectors $\{z_p(x)
  \}_{j=1}^P$ of the full input vector $x$.  Each vector $z_p(x)
  \in \real^{d_1}$ is referred to as a \emph{patch}, and these patches
    may depend on overlapping components of $x$.
\item Second, given some choice of activation function $\sigma: \real
  \rightarrow \real$ and a collection of weight vectors
  $\{w_j\}_{j=1}^r$ in $\real^{d_1}$, we compute the functions
\begin{align}
\label{eqn:def-filter}
h_j(z) \defeq \sigma(w_j^\top z)\quad \mbox{for each patch $z \in
  \R^{d_1}$.}
\end{align}
    Each function $h_j$ (for $j \in [r]$) is known as a \emph{filter}, and note that the
same filters are applied to each patch---this corresponds to the
\emph{parameter sharing} of a CNN.
\item Third, for each patch index $p \in [P]$, filter index $j \in
  [r]$, and output coordinate $\outindex \in [d_2]$, we introduce a
  coefficient $\alpha_{\outindex,j,p} \in \real$ that governs the
  contribution of the filter $h_j$ on patch $z_p(x)$ to output
  $f_\outindex(x)$.  The final form of the CNN is given by $f(x) \defn
  (f_1(x), \dots, f_{d_2}(x))$, where the $\outindex^{th}$ component
  is given by
\begin{align}
\label{eqn:def-cnn}
f_\outindex(x) \defeq \sum_{j=1}^{r} \sum_{p=1}^{P}
\alpha_{\outindex,j,p} h_j(z_p(x)).
\end{align}
\end{itemize}

Taking the patch functions $\{z_p\}_{p=1}^P$ and activation function
$\sigma$ as fixed, the parameters of the CNN are the filter vectors
$\wbold \defeq \{ w_j \in \real^{d_1} :\; j \in [r] \}$ along
with the collection of coefficient vectors \mbox{$\alphabold \defeq \{
  \alpha_{\outindex,j} \in \real^{P}:\; \outindex \in [d_2], j
  \in [r]\}$.}  We assume that all patch vectors $z_p(x) \in \R^{d_1}$
are contained in the unit $\ell_2$-ball. This assumption can be satisfied
without loss of generality
by normalization:~By multiplying a constant $\gamma > 0$ to every patch $z_p(x)$
and multiplying $1/\gamma$ to the filter vectors $\wbold$, the assumption
will be satisfied without changing the
the output of the network. 

Given some positive radii
$\bou_1$ and $\bou_2$, we consider the model class
\begin{align}
\cnn(\bou_1, \bou_2) & \defeq \Big \{ \mbox{$f$ of the
  form~\eqref{eqn:def-cnn}} :\; \mbox{$\max \limits_{j \in [r]}
  \|w_j\|_2 \leq \bou_1$ and $\max \limits_{\outindex \in [d_2], j \in
    [r]} \|\alpha_{\outindex, j}\|_2 \leq \bou_2$} \Big \}.
\end{align}
When the radii $(\bou_1, \bou_2)$ are clear from context, we adopt
$\cnn$ as a convenient shorthand. 

\subsection{Empirical risk minimization.}

Given an input-output pair $(x,y)$ and a CNN $f$, we let $\Loss(f(x);
y)$ denote the loss incurred when the output $y$ is predicted via
$f(x)$.  We assume that the loss function $\Loss$ is convex and $L$-Lipschitz
in its first argument given any value of its second argument.  As a concrete example, for multiclass
classification with $d_2$ classes, the output vector $y$ takes values
in the discrete set $[d_2] = \{1, 2, \ldots, d_2 \}$.  For example, given a vector
$f(x) = (f_1(x), \ldots, f_{d_2}(y)) \in \real^{d_2}$ of
classification scores, the associated multiclass logistic loss for a pair $(x,y)$
is given by $\Loss(f(x); y) \defeq - f_y(x) + \log \big(
\sum_{y'=1}^{d_2} \exp(f_{y'}(x))\big)$.

Given $n$ training examples $\{(x_i,y_i)\}_{i=1}^n$, we would like to
compute an empirical risk minimizer
\begin{align}
\label{eqn:empirical-risk}
\fcnn \in \arg \min_{f \in \cnn} \sum_{i=1}^n \Loss(f(x_i); y_i).
\end{align}
Recalling that functions $f \in \cnn$ depend on the parameters
$\wbold$ and $\alphabold$ in a highly nonlinear
way~\eqref{eqn:def-cnn}, this optimization problem is nonconvex.  As
mentioned earlier, heuristics based on stochastic
gradient methods are used in practice, which makes it challenging to gain a
theoretical understanding of their behavior.  Thus, in the
next section, we describe a relaxation of the class $\cnn$ that allows
us to obtain a convex formulation of the associated empirical risk
minimization problem.

\section{Convexifying CNNs}
\label{sec:single-layer-algorithm}

We now turn to the development of the class of convexified CNNs.  We
begin in Section~\ref{sec:linear} by illustrating the procedure for
the special case of the linear activation function.  Although the
linear case is not of practical interest, it provides intuition for
our more general convexification procedure, described in
Section~\ref{sec:nonlinear}, which applies to nonlinear activation
functions.  In particular, we show how embedding the nonlinear problem into an
appropriately chosen reproducing kernel Hilbert space (RKHS) allows us
to again reduce to the linear setting.


\subsection{Linear activation functions: low rank relaxations}
\label{sec:linear}

In order to develop intuition for our approach, let us begin by
considering the simple case of the linear activation function
$\sigma(t) = t$.  In this case, the filter $h_j$ when applied to the
patch vector $z_p(x)$ outputs a Euclidean inner product of the form
$h_j(z_p(x)) = \inprod{z_p(x)}{w_j}$.  For each $x \in \real^{d_0}$,
we first define the $P \times d_1$-dimensional matrix
\begin{align}
\label{EqnDefnZ}
Z(x) & \defeq \begin{bmatrix} z_1(x)^\top \\ \vdots \\ z_P(x)^\top
\end{bmatrix}.
\end{align}
We also define the $P$-dimensional vector $\alpha_{\outindex,j} \defeq
(\alpha_{\outindex,j,1},\dots,\alpha_{\outindex,j,P})^\top$.  With
this notation, we can rewrite equation \eqref{eqn:def-cnn} for the
$\outindex^{th}$ output as
\begin{align}
\label{eqn:simple-model-output}
f_\outindex(x) = \sum_{j=1}^r \sum_{p=1}^P \alpha_{\outindex,j,p}
\inprod{z_p(x)}{ w_j} = \sum_{j=1}^r \alpha_{\outindex,j}^\top Z(x)
w_j = \tr\Big(Z(x) \Big(\sum_{j=1}^r w_j \alpha_{\outindex,j}^\top
\Big)\Big) = \tr(Z(x) A_\outindex),
\end{align}
where in the final step, we have defined the $d_1 \times
P$-dimensional matrix $A_\outindex \defeq \sum_{j=1}^r w_j
\alpha_{\outindex,j}^\top$.  Observe that $f_\outindex$ now depends
linearly on the matrix parameter $A_\outindex$.  Moreover, the matrix
$A_\outindex$ has rank at most $r$, due to the parameter sharing of
CNNs. See Figure~\ref{fig:framework} for a graphical illustration of
this model structure.

\begin{figure}[t]
\begin{center}
\includegraphics[scale=0.5]{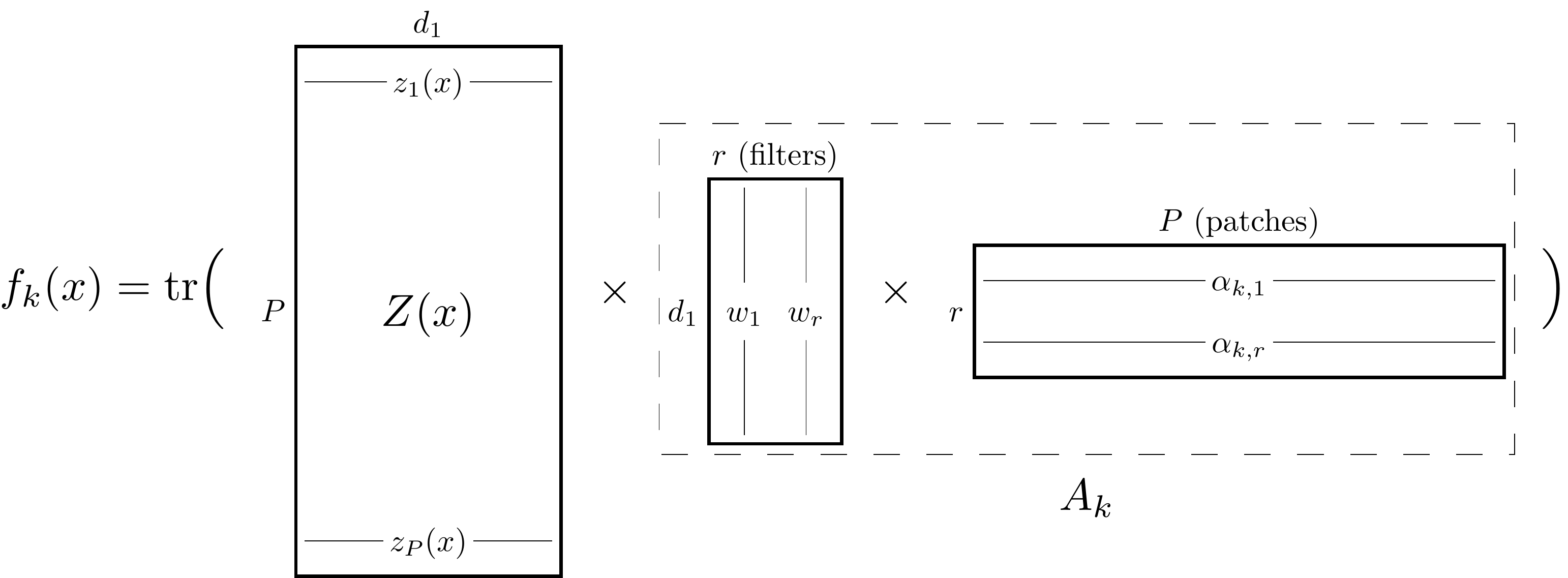}
\end{center}
\caption{\label{fig:framework}
  The $\outindex^{th}$ output of a
  CNN $f_\outindex(x) \in \real$ can be expressed as the product
  between a matrix $Z(x) \in \real^{P \times d_1}$ whose rows are
  features of the input patches and a rank-$r$ matrix $A_\outindex \in
  \real^{d_1 \times P}$, which is made up of the filter weights
  $\{w_j\}$ and coefficients $\{a_{\outindex,j,p}\}$, as illustrated.
  Due to the parameter sharing intrinsic to CNNs, the matrix
  $A_\outindex$ inherits a low rank structure, which can be enforced
  via convex relaxation using the nuclear norm.}
\end{figure}

Letting $A \defeq (A_1, \dots, A_{d_2})$ be a concatenation of these
matrices across all $d_2$ output coordinates, we can then define a
function $f^A: \real^{d_1} \rightarrow \real^{d_2}$ of the form
\begin{align}
\label{EqnFAform}
f^A(x) & \defeq (\tr(Z(x) A_1), \dots, \tr(Z(x) A_{d_2})).
\end{align}
Note that these functions have a linear parameterization in terms of
the underlying matrix $A$.  Our model class corresponds to a
collection of such functions based on imposing certain constraints on
the underlying matrix $A$: in particular, we define
\begin{align*}
\cnn(\bou_1, \bou_2) & \defeq \Big\{ f^A :\;
\underbrace{\mbox{$\max \limits_{j \in [r]} \ltwos{w_j} \leq B_1$ and
    $\max \limits_{\substack{\outindex \in [d_2] \\ j \in [r]}}
    \ltwos{\alpha_{\outindex,j}} \leq B_2$}}_{\mbox{Constraint (C1)}}
\quad \mbox{and} \underbrace{\mbox{$\text{rank}(A) =
    r$}}_{\mbox{Constraint (C2)}} \Big \}.
\end{align*}
This is simply an alternative formulation of our original class of CNNs.
Notice that if the filter weights $w_j$ are not shared across
all patches, then the constraint (C1) still holds, but constraint (C2)
no longer holds. Thus, the parameter sharing of CNNs is realized by
the low-rank constraint (C2).  The matrix $A$ of rank $r$ can be
decomposed as $A = U V^\top$, where both $U$ and $V$ have $r$
columns. The column space of matrix $A$ contains the convolution
parameters $\{w_j\}$, and the row space of $A$ contains to the output
parameters $\{\alpha_{\outindex,j}\}$.

The rank-$r$ matrices satisfying constraints (C1) and (C2) form a
nonconvex set.  A standard convex relaxation of a rank constraint is
based on the nuclear norm $\nucnorm{A}$ corresponding to the sum of the singular values of $A$.  It
is straightforward to verify that any matrix $A$ satisfying the
constraints (C1) and (C2) must have nuclear norm bounded as
$\nucnorm{A} \leq B_1 B_2 r \sqrt{d_2}$.  Consequently, if we define
the function class
\begin{align}
\label{eqn:simple-model-nuclear-constraint}
\rcnn & \defeq \Big \{ f^{A} :\; \nucnorm{A} \leq B_1 B_2 r
\sqrt{d_2} \Big \},
\end{align}
then we are guaranteed that $\rcnn \supseteq \cnn$.  

Overall, we propose to minimize the empirical
risk~\eqref{eqn:empirical-risk} over $\rcnn$ instead of $\cnn$; doing
so defines a convex optimization problem over this richer class of
functions
\begin{align}
  \label{eqn:obj-func}
\frcnn \defeq \arg\min_{f^A \in \rcnn} \sum_{i=1}^n \Loss(f^A(x_i); y_i).
\end{align}
In Section~\ref{sec:algorithm}, we describe iterative algorithms that
can be used to solve this form of convex program in the more general
setting of nonlinear activation functions.


\subsection{Nonlinear activations: RKHS filters}
\label{sec:nonlinear}

For nonlinear activation functions $\sigma$, we relax the class of
CNN filters to a reproducing kernel Hilbert space (RKHS).  As we show,
this relaxation allows us to reduce the problem to the linear
activation case.

Let $\kf : \R^{d_1} \times \R^{d_1} \to \R$ be a positive semidefinite
kernel function. For particular choices of kernels (e.g.,~the Gaussian RBF kernel) and
some sufficiently smooth activation function~$\sigma$, we are able to show that
the filter $h:z \mapsto \sigma (\inprod{w}{z})$ is contained in the RKHS induced by the kernel
function $\kf$. See Section~\ref{sec:theory} for the choice of the kernel function and the activation function. Let $S \defeq \{ z_p(x_i) : p \in [P], i \in [n] \}$
be the set of patches in the training dataset.
The representer theorem then implies that for any patch $z_p(x_i) \in S$,
the function value can be represented by
\begin{align}
\label{eqn:representer-theorem-finite-basis}
h(z_p(x_i)) = \sum_{(i',p')\in [n]\times [P]} c_{i',p'} k(z_p(x_i),z_{p'}(x_{i'}))
\end{align}
for some coefficients $\{c_{i',p'}\}_{(i',p')\in [n]\times [P]}$.
Filters taking the
form~\eqref{eqn:representer-theorem-finite-basis} are members of the RKHS, because they are linear combinations of basis functions $z\mapsto k(z,z_{p'}(x_{i'}))$. 
Such filters are parametrized by a finite set of coefficients, which can be
estimated via empirical risk minimization.


Let $K \in \R^{n P\times n P}$ be the symmetric kernel matrix, where
with rows and columns indexed by the example-patch index pair $(i,p)
\in [n] \times [P]$.  The entry at row $(i,p)$ and column $(i',p')$ of
matrix $K$ is equal to $\kf(z_p(x_i), z_{p'}(x_{i'}))$.  So as to
avoid re-deriving everything in the kernelized setting, we perform a
reduction to the linear setting of Section~\ref{sec:linear}.  Consider
a factorization $K = Q Q^\top$ of the kernel matrix, where $Q \in
\R^{nP\times m}$; one example is the Cholesky factorization with $m =
nP$. We can interpret each row $Q_{(i,p)} \in \R^m$ as a feature
vector in place of the original $z_p(x_i) \in \R^{d_1}$, and rewrite
equation~\eqref{eqn:representer-theorem-finite-basis} as
\begin{align*}
h(z_p(x_i)) = \langle Q_{(i,p)}, w \rangle \quad \mbox{where}
\quad w \defeq \sum_{(i',p')} c_{i',p'} Q_{(i',p')}.
\end{align*}
In order to learn the filter $h$, it suffices to learn the
$m$-dimensional vector $w$.  To do this, define patch matrices
$Z(x_i)\in \R^{P\times m}$ for each $i\in [n]$ so that its $p$-th
row is $Q_{(i,p)}$.
Then we carry out all of Section~\ref{sec:linear};
solving the ERM gives us a parameter matrix $A \in \R^{m \times P
  d_2}$.  The only difference is that the $\bou_1$ norm constraint
needs to be relaxed as well. See
Appendix~\ref{sec:nonlinear-activation} for details.

At test time, given a new input $x \in \R^{d_0}$, we can compute a
patch matrix $Z(x) \in \R^{P \times m}$ as follows:
\begin{itemize}[leftmargin=*]
\item The $p$-th row of this matrix is the feature vector for patch
  $p$, which is equal to $Q^\dagger v(z_p(x)) \in \R^m$. Here, for any
  patch $z$, the vector $v(z)$ is defined as a $nP$-dimensional vector
  whose $(i,p)$-th coordinate is equal to $\kf(z,z_p(x_i))$. We note
	that if $x$ is an instance $x_i$ in the training set, then the
	vector $Q^\dagger v(z_p(x))$ is exactly equal to $Q_{(i,p)}$. Thus 
	the mapping $Z(x)$ applies to both training and testing.
	
\item We can then compute the predictor $f_\outindex(x) = \tr(Z(x)
  A_\outindex)$ via equation~\eqref{eqn:simple-model-output}. Note
  that we do not explicitly need to compute the filter values
  $h_j(z_p(x))$ to compute the output under the CCNN.
\end{itemize}

\myparagraph{Retrieving filters} However, when we learn multi-layer
CCNNs, we need to compute the filters explicitly.  Recall from
Section~\ref{sec:linear} that the column space of matrix $A$
corresponds to parameters of the convolutional layer, and the row space
of $A$ corresponds to parameters of the output layer. Thus, once we
obtain the parameter matrix $A$, we compute a rank-$r$
  approximation $A \approx \widehat U \widehat V^\top$. Then
set the $j$-th filter $h_j$ to the mapping
\begin{align}
\label{eqn:filter-recovery}
z \mapsto \inprod{ \widehat U_j}{Q^\dagger v(z)} \quad \mbox{for any patch $z\in
  \R^{d_1}$,}
\end{align}
where $\widehat U_j \in \R^m$ is the $j$-th column of matrix $\widehat U$, and
$Q^\dagger v(z)$ represents the feature vector for patch~$z$. The matrix $\widehat V^\top$ encodes parameters of the output layer, thus 
doesn't appear in the filter expression~\eqref{eqn:filter-recovery}. It is important to note that the filter retrieval is not unique, because the rank-$r$ approximation of the matrix $A$ is not unique. One feasible way is to form the singular value decomposition \mbox{$A =  U \Lambda V^\top$}, then define $\widehat U$ to be the first $r$ columns of $U$, and define $\widehat V^\top$ to be the first $r$ rows of $\Lambda V^\top$. 

When we apply all of the $r$ filters to all patches of an input $x\in \R^{d_0}$,
the resulting output is $H(x) \defeq \widehat U^\top (Z(x))^\top$ --- this is an $r \times P$
matrix whose element at row $j$ and column $p$ is equal
to~$h_j(z_p(x))$.


\subsection{Algorithm}
\label{sec:algorithm}

\begin{algorithm}[t]
\begin{flushleft}
{\bf Input: }{Data $\{(x_i,y_i)\}_{i=1}^n$, kernel function $\kf$,
  regularization parameter $R > 0$, number of filters $r$.}
\begin{enumerate}
\item Construct a kernel matrix $K\in \R^{nP\times nP}$ such that the
  entry at column $(i,p)$ and row $(i',p')$ is equal to $\kf(z_p(x_i),
  z_{p'}(x_{i'}))$. Compute a factorization $K = QQ^\top$ or an
  approximation $K \approx QQ^\top$, where $Q\in \R^{nP\times m}$.
\item For each $x_i$, construct patch matrix $Z(x_i)\in \R^{P\times
  m}$ whose $p$-th row is the $(i,p)$-th row of $Q$,
  where $Z(\cdot)$ is defined in Section~\ref{sec:nonlinear}.
\item Solve the following optimization problem to obtain a matrix
  $\Ahat = (\Ahat_1,\dots, \Ahat_{d_2})$:
\begin{align}
\label{eqn:alg-obj-func}
\Ahat & \in \underset{\nucnorm{A}\leq R}{\rm argmin}\; \LossTil(A)
\mbox{~~where~~}
\LossTil(A) \defeq \sum_{i=1}^n  \Loss\Big(\big(\tr(Z(x_i)A_1),
\dots, \tr(Z(x_i)A_{d_2})\big); y_i\Big).
\end{align}
\item Compute a rank-$r$ approximation $\Atil \approx \widehat U \widehat V^\top$
where $\widehat U\in \R^{m \times r}$ and $\widehat V \in \R^{P d_2 \times
    r}$.
\end{enumerate}

{\bf Output:} Return the predictor $\frcnn(x) \defeq
\big(\tr(Z(x)\Ahat_1), \dots, \tr(Z(x)\Ahat_{d_2})\big)$ and the
convolutional layer output $H(x) \defeq \widehat U^\top (Z(x))^\top$.
\end{flushleft}
\caption{Learning two-layer CCNNs}\label{alg:two-layer-ccnn}
\end{algorithm}

The algorithm for learning a two-layer CCNN is summarized in
Algorithm~\ref{alg:two-layer-ccnn}; it is a formalization of the steps
described in Section~\ref{sec:nonlinear}. In order to solve the
optimization problem~\eqref{eqn:alg-obj-func}, the simplest approach
is to via projected gradient descent:~At iteration $t$, using a step
size $\etastep{\step} > 0$, it forms the new matrix $\Amat{\step +1}$
based on the previous iterate $\Amat{\step}$ according to:
\begin{align}
\Amat{\step+1} = \ProjR \Big( \Amat{\step} - \etastep{\step} \;
\nabla_A \LossTil (\Amat{\step}) \Big).
\end{align}
Here $\nabla_A \LossTil$ denotes the gradient of the objective
function defined in~\eqref{eqn:alg-obj-func}, and $\ProjR$ denotes the Euclidean projection
onto the nuclear norm ball $\{A : \lncs{A} \leq R\}$.  This
nuclear norm projection can be obtained by first computing the
singular value decomposition of $A$, and then projecting the vector of
singular values onto the $\ell_1$-ball. This latter projection step
can be carried out efficiently by the algorithm
of~\citet{duchi2008efficient}. There are other efficient optimization
algorithms for solving the problem~\eqref{eqn:alg-obj-func}, such as the
proximal adaptive gradient method~\cite{duchi2011adaptive} and the
proximal SVRG method~\cite{xiao2014proximal}. All these algorithms can
be executed in a stochastic fashion, so that each gradient step
processes a mini-batch of examples.

The computational complexity of each iteration depends on the width
$m$ of the matrix $Q$.  Setting $m = nP$ allows us to solve the exact
kernelized problem, but to improve the computation efficiency, we can
use Nystr{\"o}m approximation~\cite{drineas2005nystrom} or random
feature approximation~\cite{rahimi2007random}; both are randomized
methods to obtain a tall-and-thin matrix $Q \in \R^{nP\times m}$ such
that $K \approx Q Q^\top$. Typically, the parameter $m$ is chosen to
be much smaller than $n P$. In order to compute the matrix $Q$, the
Nystr{\"o}m approximation method takes $\order(m^2 nP)$ time. The
random feature approximation takes $\order(m n P d_1)$ time, but can
be improved to $\order(m n P \log d_1)$ time using the fast Hadamard
transform~\cite{le2013fastfood}.  The complexity of computing a
gradient vector on a batch of $b$ images is $\order(m P d_2 b)$. The
complexity of projecting the parameter matrix onto the nuclear norm
ball is $\order(\min\{m^2 P d_2, m P^2 d_2^2\})$.  Thus, the
approximate algorithms provide substantial speed-ups on the projected
gradient descent steps. 


\subsection{Theoretical results}
\label{sec:theory}

In this section, we upper bound the generalization error of Algorithm~\ref{alg:two-layer-ccnn}, proving that it converges to the best possible generalization error of CNN. 
We focus on the binary classification case where the output dimension is $d_2 = 1$.\footnote{We can treat
the multiclass case by performing a standard one-versus-all reduction to the binary case.}

The learning of CCNN requires a kernel function $\kf$.
We consider kernel functions whose associated RKHS is large enough to contain any function taking the following form: $z \mapsto q(\langle w, z\rangle)$, where $q$ is an arbitrary
polynomial function and $w\in \R^{d_1}$ is an arbitrary vector. As a concrete example, we consider the inverse polynomial kernel:
\begin{align}\label{eqn:def-ipk}
\kf(z,z') & \defeq \frac{1}{2 - \inprod{z}{z'}},\qquad \ltwos{z}\leq 1, \ltwos{z'}\leq 1.
\end{align}
This kernel was 
studied by \citet{shalev2011learning} for learning
halfspaces, and by
\citet{zhang2015ell_1} for learning fully-connected neural networks.
We also consider the Gaussian RBF kernel:
\begin{align}\label{eqn:def-gaussian-kernel}
\kf(z,z') & \defeq \exp(- \gamma \ltwos{z-z'}^2),\qquad \ltwos{z}=\ltwos{z'}= 1, \gamma > 0.
\end{align}
As we show in Appendix~\ref{sec:ipk-and-gaussian-kernel}, the inverse polynomial kernel and the Gaussian kernel satisfy the above notion of richness. We focus on these two kernels for the theoretical analysis.

Let $\frcnn$ be the CCNN that minimizes the empirical
risk~\eqref{eqn:alg-obj-func} using one of the two kernels above.  Our
main theoretical result is that for suitably chosen activation
functions, the generalization error of $\frcnn$ is comparable to that
of the best CNN model.  In particular, the following theorem applies
to activation functions $\sigma$ of the following types:
\begin{enumerate}[leftmargin=*, label={(\alph*)}]
\item arbitrary polynomial functions (e.g., used
  by~\cite{chen2014fast,livni2014computational}).\label{label:poly-func}
\item sinusoid activation function $\sigma(t) \defeq \sin(t)$ (e.g.,
  used
  by~\cite{sopena1999neural,isa2010suitable}).\label{label:sine-func}
\item erf function $\erf(t) \defeq 2/\sqrt{\pi}\int_{0}^t e^{-z^2}
  dz$, which represents an approximation to the sigmoid function (See
  Figure~\ref{fig:compare-activation}(a)). \label{label:erf-func}
\item a smoothed hinge loss $\sigma_{\rm sh}(t) \defeq
  \int_{-\infty}^t \frac{1}{2}(\erf(z)+1) dz$, which represents an
  approximation to the ReLU function (See
  Figure~\ref{fig:compare-activation}(b)).\label{label:hinge-func}
\end{enumerate}

To understand why these activation functions pair with our choice of
kernels, we consider polynomial expansions of the above activation functions: $\sigma(t) = \sum_{j=0}^\infty a_j t^j$, and note that the smoothness of these functions are characterized by the rate of their coefficients $\{a_j\}_{j=0}^\infty$ converging to zero. If $\sigma$ is a polynomial in category~\ref{label:poly-func}, then the richness of the RKHS guarantees that it contains the class of filters activated by function $\sigma$. If $\sigma$ is a non-polynomial function in categories~\ref{label:sine-func},\ref{label:erf-func},\ref{label:hinge-func}, then as Appendix~\ref{sec:ipk-and-gaussian-kernel} shows, the RKHS contains the filter only if the coefficients $\{a_j\}_{j=0}^\infty$ converge quickly enough to zero (the criterion depends on the choice of the kernel). Concretely, the inverse polynomial kernel is shown to capture all of the four categories of activations: they are referred as \emph{valid activation functions} for the inverse
polynomial kernel. The Gaussian kernel induces a smaller RKHS, and is shown to capture categories~\ref{label:poly-func},\ref{label:sine-func}, so that these
functions are referred as \emph{valid activation functions} for the Gaussian kernel.
In contrast,
the sigmoid function and the ReLU function are not valid for either
kernel, because their polynomial expansions fail to converge quickly
enough, or more intuitively speaking, because they are not smooth
enough functions to be contained in the RKHS.

We are ready to state the main theoretical result. 
In the theorem statement, we use $K(X)\in \R^{P\times P}$ to denote
the random kernel matrix obtained from an input vector $X \in
\real^{d_0}$ drawn randomly from the population.  More precisely, the
$(p,q)$-th entry of $K(X)$ is given by $\kf(z_p(X), z_q(X))$.

\begin{theorem}
\label{theorem:single-layer-rcnn}
Assume that the loss function $\Loss(\cdot;y)$ is $L$-Lipchitz continuous for every $y\in[d_2]$ and that $\kf$ is the inverse polynomial kernel or the Gaussian kernel.
For any valid activation function $\sigma$, there is a constant
$C_\sigma(B_1)$ such that with the radius $R \defeq C_\sigma(B_1)B_2
r$, the expected generalization error is at most
\begin{align}
\E_{X,Y}[\Loss(\frcnn(X); Y)] \leq \inf_{f\in \cnn}\E_{X,Y}
  [\Loss(f(X); Y)] + \frac{c\; L C_\sigma(B_1)B_2 r \sqrt{\log(n P) \;
      \E_X[\ltwos{K(X)}]}}{\sqrt{n} },
\end{align}
where $c > 0$ is a universal constant.
\end{theorem}

\paragraph{Proof sketch}

The proof of Theorem~\ref{theorem:single-layer-rcnn} consists of two parts: First, we consider a larger function class that contains the class of CNNs. This function class is defined as:
\begin{align*}
\rcnn \defeq \Big\{ x\mapsto \sum_{j=1}^{r^*} \sum_{p=1}^{P} \alpha_{j,p}
h_j(z_p(x)) :\; r^* < \infty \mbox{ and } \sum_{j=1}^{r^*}
\ltwos{\alpha_j} \norms{h_j}_{\Hset} \leq C_\sigma(B_1) B_2 d_2
\Big\}.
\end{align*}
where $\hnorms{\cdot}$ is the norm of the RKHS associated with the kernel. This new function class relaxes the class of CNNs in two ways: 1) the filters are relaxed to belong to the RKHS, and 2) the $\ell_2$-norm bounds on the weight vectors are replaced by a single constraint on $\ltwos{\alpha_j}$ and $\hnorms{h_j}$.
We prove the following property for the predictor $\frcnn$: it must be an empirical risk minimizer of $\rcnn$, even though the algorithm has never explicitly optimized the loss within this nonparametric function class.

Second, we characterize the Rademacher complexity of this new function class $\rcnn$, proving an upper bound for it based on the matrix concentration theory. 
Combining this bound with the classical Rademacher complexity theory~\cite{bartlett2003rademacher}, we conclude
that the generalization loss of $\frcnn$ converges to the least possible generalization error of $\rcnn$. The later loss is bounded by the generalization loss of CNNs (because $\cnn\subseteq \rcnn$), which establishes the theorem. See
Appendix~\ref{sec:proof-main-theorem} for the full proof of Theorem~\ref{theorem:single-layer-rcnn}. 

\begin{figure}
\centering
\begin{tabular}{ccc}
\includegraphics[width = 0.4\textwidth]{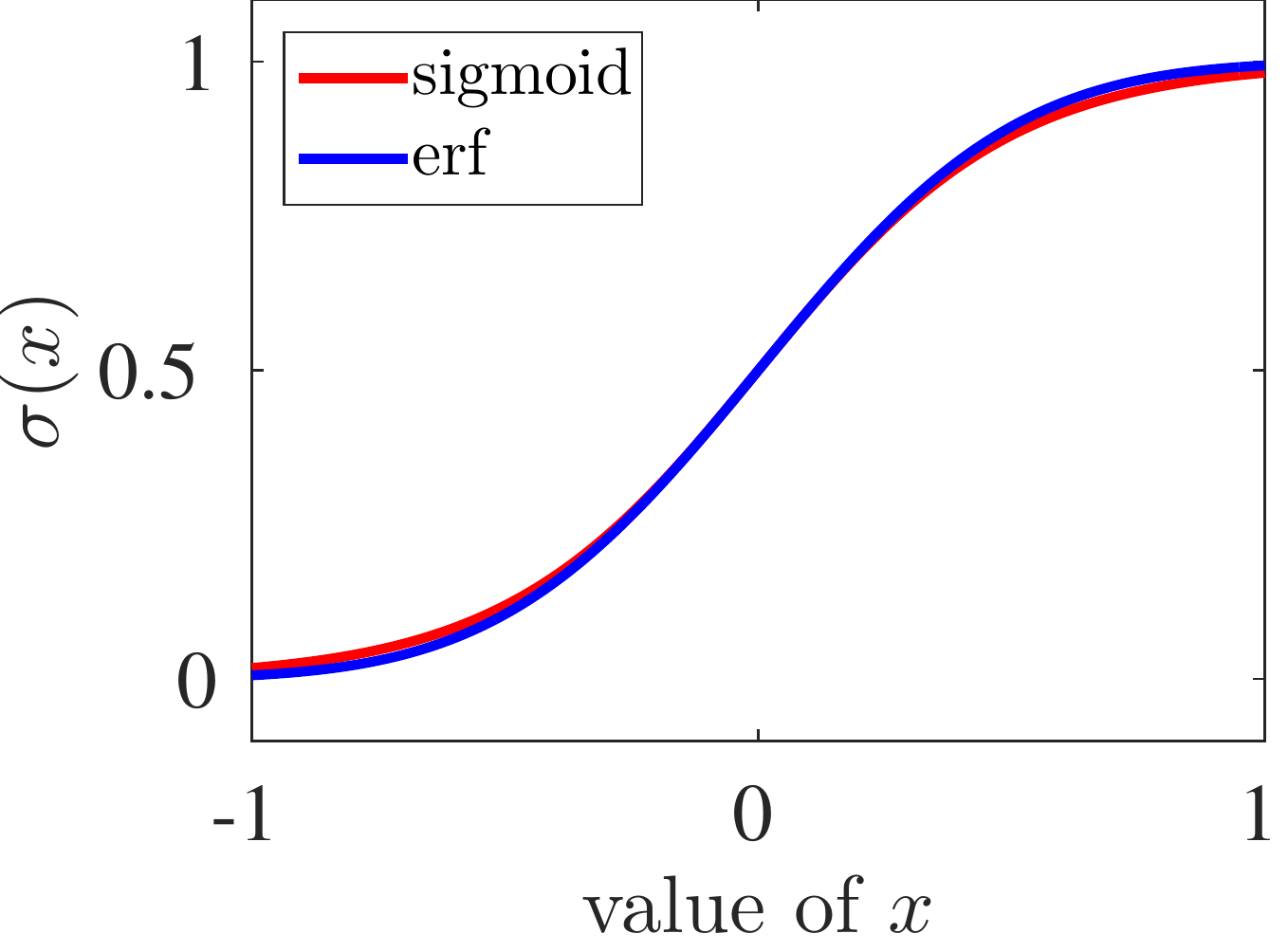} & &
\includegraphics[width = 0.4\textwidth]{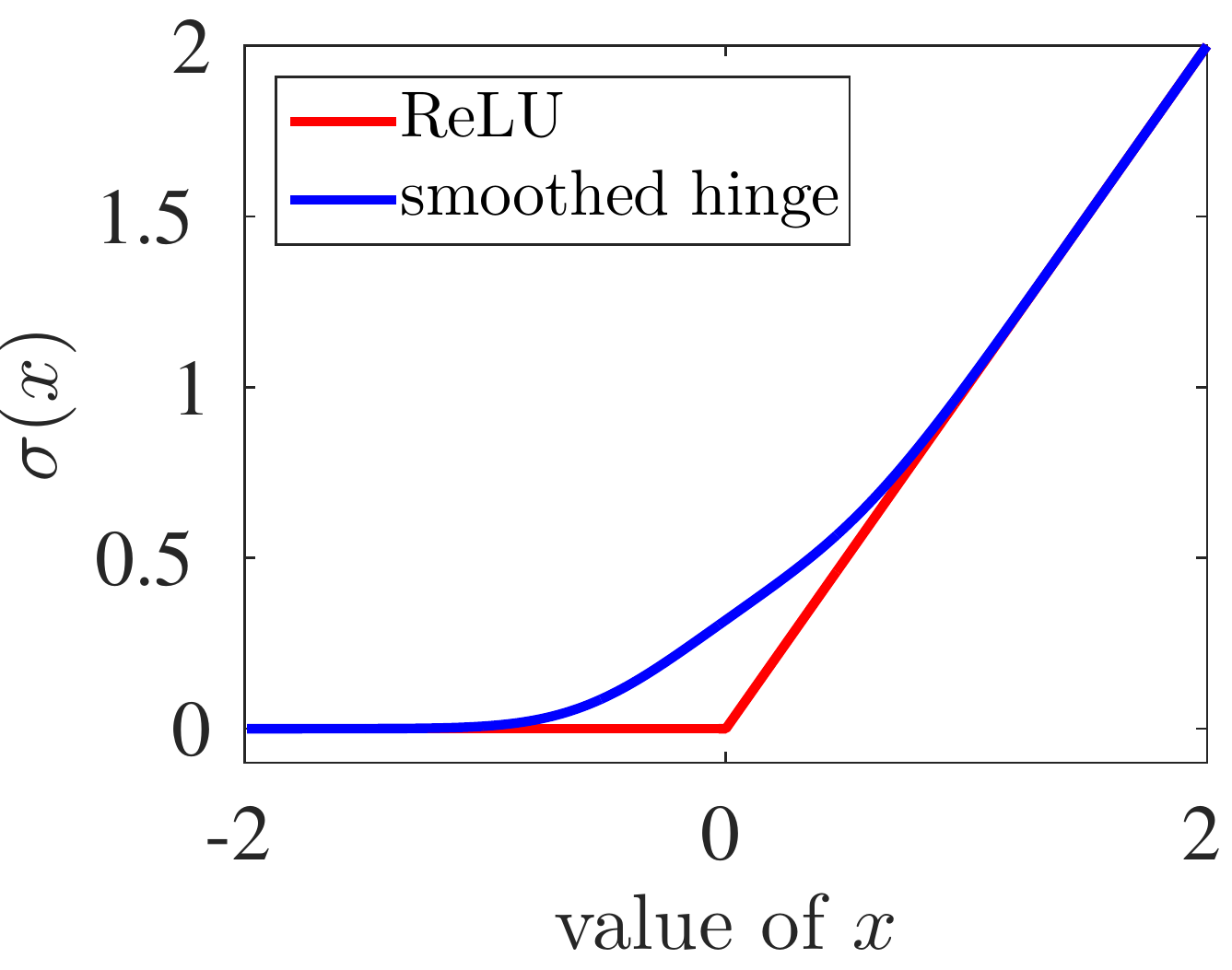}\\ (a) sigmoid
v.s. erf && (b) ReLU v.s. smoothed hinge loss
\end{tabular}
\caption{Comparing different activation functions. The two functions
  in (a) are quite similar. The smooth hinge loss in (b) is a smoothed
  version of ReLU.}\label{fig:compare-activation}
\end{figure}


\myparagraph{Remark on activation functions}

It is worth noting that the quantity $C_\sigma(B_1)$ depends on the
activation function $\sigma$, and more precisely, depends on the
convergence rate of the polynomial expansion of $\sigma$.
Appendix~\ref{sec:ipk-and-gaussian-kernel} shows that if $\sigma$ is a
polynomial function of degree $\ell$, then $C_\sigma(B_1) =
\order(B_1^\ell)$.  If $\sigma$ is the sinusoid function, the erf
function or the smoothed hinge loss, then the quantity $C_\sigma(B_1)$
will be exponential in~$B_1$. In an algorithmic perspective, we don't
need to know the activation function for executing
Algorithm~\ref{alg:two-layer-ccnn}. In a theoretical perspective,
however, the choice of $\sigma$ is relevant from the point of
Theorem~\ref{theorem:single-layer-rcnn} to compare $\frcnn$ with the
best CNN, whose representation power is characterized by the choice of
$\sigma$.  Therefore, if a CNN with a low-degree polynomial $\sigma$
performs well on a given task, then CCNN also enjoys correspondingly
strong generalization.  Empirically, this is actually borne
out: in Section~\ref{sec:experiment}, we show that the quadratic
activation function performs almost as well as the ReLU function for
digit classification. 

\myparagraph{Remark on parameter sharing}

In order to demonstrate the importance of parameter sharing, consider
a CNN without parameter sharing, so that we have filter weights
$w_{j,p}$ for each filter index $j$ and patch index $p$.  With this
change, the new CNN output~\eqref{eqn:def-cnn} is
\begin{align}
f(x) = \sum_{j=1}^{r} \sum_{p=1}^P \alpha_{j,p} \sigma(w_{j,p}^\top
z_p(x)),\quad \mbox{where $\alpha_{j,p}\in \R$ and $w_{j,p}\in
  \R^{d_1}$}.
\end{align}
Note that the hidden layer of this new network has $P$ times more
parameters than that of the convolutional neural network with parameter
sharing.  These networks without parameter sharing can be learned by the recursive kernel method
proposed by~\citet{zhang2015ell_1}. This paper shows that under the norm constraints $\ltwos{w_j} \leq \bou'_1$ and
$\sum_{j=1}^{r} \sum_{p=1}^P |\alpha_{j,p}| \leq \bou'_2$, the excess
risk of the recursive kernel method is at most $\order(L
C_\sigma(B'_1)B'_2\sqrt{ K_{\max} / n})$, where $K_{\max} = \max_{z:
  \ltwos{z}\leq 1} \kf(z,z)$ is the maximal value of the kernel
function.  Plugging in the norm constraints of the function class
$\cnn$, we have $B'_1 = B_1$ and $B'_2 = B_2 r \sqrt{P}$.  Thus, the
expected risk of the estimated $\fhat$ is bounded by:
\begin{align}
\E_{X,Y}[\Loss(\fhat(X); Y)] \leq \inf_{f\in \cnn}\E_{X,Y}[\Loss(f(X);
  Y)] + \frac{c\; L C_\sigma(B_1) B_2 r \sqrt{P K_{\max}}}{\sqrt{n}}.
\end{align}
Comparing this bound to Theorem~\ref{theorem:single-layer-rcnn}, we
see that (apart from the logarithmic terms) they differ in the
multiplicative factors of $\sqrt{P \, K_{\max}}$ versus
$\sqrt{\E[\ltwos{K(X)}]}$.  Since the matrix $K(X)$ is
\mbox{$P$-dimensional,} we have
\begin{align*}
\ltwos{K(X)} & \leq \max_{p \in [P]} \sum_{q \in [P]} |\kf(z_p(X),
z_q(X))| \leq P \, K_{\max}.
\end{align*}
This demonstrates that $\sqrt{P \, K_{\max}}$ is always greater than $\sqrt{\E[\ltwos{K(X)}]}$.
In general, the first term can be up to factor of $\sqrt{P}$ times
greater, which implies that the sample complexity of the recursive
kernel method is up to $P$ times greater than that of the CCNN. This
difference corresponds to the fact that the recursive kernel method
learns a model with $P$ times more parameters. Although comparing the upper bounds doesn't 
rigorously show that one method is better than the other, it gives the right intuition for understanding the importance of parameter sharing.

\section{Learning multi-layer CCNNs}
\label{sec:multi}

In this section, we describe a heuristic method for learning CNNs with more
layers. The idea is to estimate the parameters of the convolutional layers incrementally
from bottom to the top. Before presenting the multi-layer algorithm,
we present two extensions, average pooling and multi-channel inputs.

\myparagraph{Average pooling}

Average pooling is a technique to reduce the output dimension of the
convolutional layer from dimensions $P \times r$ to dimensions $P'
\times r$ with $P' < P$. Suppose that the filter $h_j$ applied to all
the patch vectors produces the output vector $H_j(x) \defeq
(h_j(z_1(x)), \cdots, h_j(z_P(x))) \in \R^{P \times r}$. Average
pooling produces a $P' \times r$ matrix, where each row is the average of
the rows corresponding to a small subset of the $P$ patches.  For example, we might average every
pair of adjacent patches, which would produce $P' = P/2$ rows.  The
operation of average pooling can be represented via
left-multiplication using a fixed matrix $G \in \R^{P' \times P}$.

For the CCNN model, if we apply average pooling after the
convolutional layer, then the $\outindex$-th output of the CCNN model
becomes $\tr(G Z(x) A_\outindex)$ where $A_\outindex \in \R^{m\times
  P'}$ is the new (smaller) parameter matrix.  Thus, performing a
pooling operation requires only replacing every matrix $Z(x_i)$ in
problem~\eqref{eqn:alg-obj-func} by the pooled matrix $G Z(x_i)$.  Note
that the linearity of the CCNN allows us to effectively pool before convolution,
even though for the CNN, pooling must be done after applying the nonlinear filters.
The resulting ERM problem is still convex, and the number of parameters
have been reduced by $P/P'$-fold.  Although average pooling is
straightforward to incorporate in our framework, unfortunately, max
pooling does not fit into our framework due to its nonlinearity.


\myparagraph{Processing multi-channel inputs}

If our input has $C$ channels (corresponding to RGB colors, for example), then the input
becomes a matrix $x \in \R^{C\times d_0}$. The $c$-th row of matrix
$x$, denoted by $x[c] \in \R^{d_0}$, is a vector representing the $c$-th
channel. We define the multi-channel patch vector as a concatenation
of patch vectors for each channel:
\begin{align*}
z_p(x) \defeq (z_p(x[1]), \dots, z_p(x[C])) \in \R^{C d_1}.
\end{align*}
Then we construct the feature matrix $Z(x)$ using the concatenated patch
vectors $\{z_p(x)\}_{p=1}^P$. From here, everything else of
Algorithm~\ref{alg:two-layer-ccnn} remains the same. We note that this
approach learns a convex relaxation of filters taking the form
$\sigma(\sum_{c=1}^C \langle w_c, z_p(x[c])\rangle)$, parametrized by
the vectors $\{w_c\}_{c=1}^C$.


\myparagraph{Multi-layer CCNN}

\begin{algorithm}[t]
\begin{flushleft}
{\bf Input:}{Data $\{(x_i,y_i)\}_{i=1}^n$, kernel function $\kf$,
  number of layers $m$, regularization parameters $R_1,\dots,R_m$,
  number of filters $r_1,\dots,r_m$.} \\
  Define $H_1(x) = x$. \\
  For each layer $s = 2, \dots, m$:
\begin{itemize}
  \item Train a two-layer network by Algorithm~\ref{alg:two-layer-ccnn},
  taking $\{(H_{s-1}(x_i),y_i)\}_{i=1}^n$ as training examples and $R_s, r_s$ as parameters.
    Let $H_s$ be the output of the convolutional layer and $\fhat_s$ be the predictor.
\end{itemize}

{\bf Output:} Predictor $\fhat_m$ and the top convolutional layer output $H_m$.
\end{flushleft}
\caption{Learning multi-layer CCNNs}\label{alg:multi-layer-ccnn}
\end{algorithm}

Given these extensions, we are ready to present the algorithm for
learning multi-layer CCNNs. The algorithm is summarized in
Algorithm~\ref{alg:multi-layer-ccnn}.
For each layer $s$, we call Algorithm~\ref{alg:two-layer-ccnn}
using the output of previous convolutional layers as input---note that
this consists of $r$ channels (one from each previous filter)
and thus we must use the multi-channel extension.
Algorithm~\ref{alg:multi-layer-ccnn}
outputs a new convolutional layer along with a prediction function,
which is kept only at the last layer.
We optionally use averaging pooling after each successive layer.
to reduce the output dimension of the convolutional layers.

\begin{figure}[t]
\centering
\begin{tabular}{ccc}
\includegraphics[width = 0.45\textwidth]{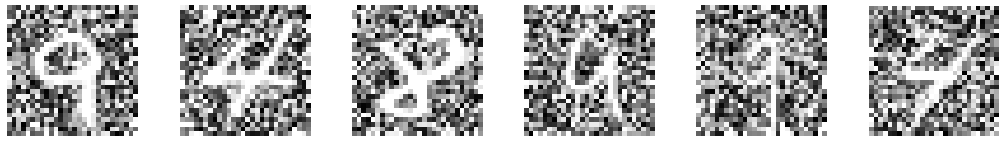}
&&
\includegraphics[width = 0.45\textwidth]{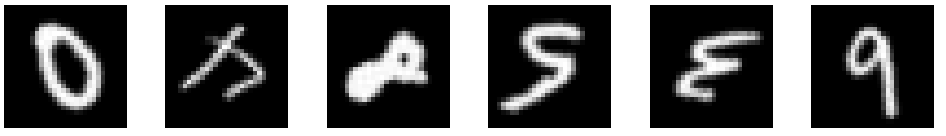}\\
(a) {\tt rand} && (b) {\tt rot}\\
&&\\
\includegraphics[width = 0.45\textwidth]{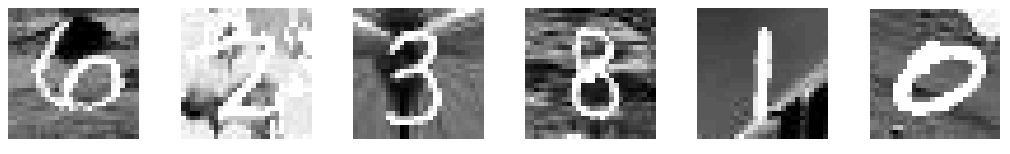}&&
\includegraphics[width = 0.45\textwidth]{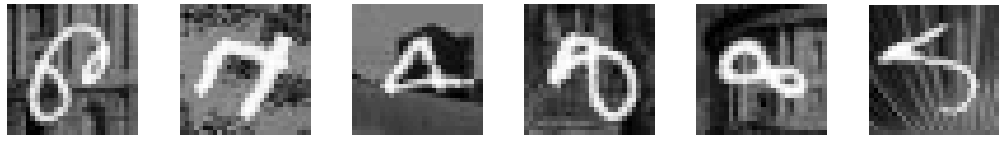}\\
(c) {\tt img} && (d) {\tt img+rot}
\end{tabular}
\caption{Some variations of the MNIST dataset: (a) random background
  inserted into the digit image; (b) digits rotated by a random angle
  generated uniformly between $0$ and $2\pi$; (c) black and white
  image used as the background for the digit image; (d) combination of
  background perturbation and rotation perturbation.}\label{fig:mnist}
\end{figure}

\section{Experiments}
\label{sec:experiment}

In this section, we compare the CCNN approach with other methods. The
results are reported on the MNIST dataset and its variations for digit
recognition, and on the CIFAR-10 dataset for object classification.


\subsection{MNIST and variations}

Since the basic MNIST digits are relatively easy to classify, we also
consider more challenging variations~\cite{WinNT}.  These variations
are known to be hard for methods without a convolution mechanism (for
instance, see the paper~\cite{vincent2010stacked}).
Figure~\ref{fig:mnist} shows a number of sample images from these
different datasets. All the images are of size $28\times 28$. For all
datasets, we use 10,000 images for training, 2,000 images for
validation and 50,000 images for testing.  This 10k/2k/50k
partitioning is standard for MNIST variations~\cite{WinNT}.

\myparagraph{Implementation details}

For the CCNN method and the baseline CNN method, we train two-layer
and three-layer models respectively. The models with $k$ convolutional
layers are denoted by CCNN-$k$ and CNN-$k$. Each convolutional layer is
constructed on $5 \times 5$ patches with unit stride, followed by
$2\times 2$ average pooling. The first and the second convolutional
layers contains 16 and 32 filters, respectively.  The loss function is
chosen as the $10$-class logistic loss. We use the Gaussian kernel
$\kf(z,z') = \exp(-\gamma \ltwos{z-z'}^2)$ and set hyperparameters
$\gamma = 0.2$ for the first convolutional layer and $\gamma = 2$ for the second. The feature
matrix $Z(x)$ is constructed via random feature
approximation~\cite{rahimi2007random} with dimension $m=500$ for the first convolutional layer
and $m=1000$ for the second.
Before training each CCNN layer, we
preprocess the input vectors $z_p(x_i)$ using local contrast
normalization and ZCA whitening~\cite{coates2010analysis}. The convex
optimization problem is solved by projected SGD with mini-batches of size $50$.

As a baseline approach, the CNN models are activated by the ReLU
function $\sigma(t) = \max\{0,t\}$ or the quadratic function
$\sigma(t) = t^2$.  We train them using mini-batch SGD. The input
images are preprocessed by global contrast normalization and ZCA
whitening~\cite[see, e.g.][]{srivastava2014dropout}.  We compare our
method against several alternative baselines. The CCNN-1 model is
compared against an SVM with the Gaussian RBF kernel (SVM$_{rbf}$) and
a fully connected neural network with one hidden layer (NN-1). The
CCNN-2 model is compared against methods that report the
state-of-the-art results on these datasets, including the
translation-invariant RBM model (TIRBM)~\cite{sohn2012learning}, the
stacked denoising auto-encoder with three hidden layers
(SDAE-3)~\cite{vincent2010stacked}, the ScatNet-2
model~\cite{bruna2013invariant} and the PCANet-2
model~\cite{chan2015pcanet}.


\myparagraph{Results}

\begin{table}[t]\small
\begin{tabular}{|l|c|c|c|c|c|}\hline
 & {\tt basic} & {\tt rand} & {\tt rot} & {\tt img} & {\tt
    img+rot}\\\hline SVM$_{ rbf}$~\cite{vincent2010stacked} & 3.03\% &
  14.58\%&{\bf 11.11\%}& 22.61\%&
  55.18\%\\ NN-1~\cite{vincent2010stacked}&4.69\% &20.04\% &18.11\%&
  27.41\%& 62.16\%\\ CNN-1 (ReLU)& 3.37\% &9.83\% &18.84\%& 14.23\%&
  45.96\% \\ CCNN-1& {\bf 2.38\%} &{\bf 7.45\%} &13.39\%& {\bf
    10.40\%}& {\bf 42.28\%}\\ \hline TIRBM~\cite{sohn2012learning} & -
  & - & {\bf 4.20\%}&-& 35.50\%\\ SDAE-3~\cite{vincent2010stacked}
  &2.84\%& 10.30\% &9.53\% &16.68\%
  &43.76\%\\ ScatNet-2~\cite{bruna2013invariant} &1.27\%& 12.30\%&
  7.48\%& 18.40\%& 50.48\%\\ PCANet-2~\cite{chan2015pcanet} & {\bf
    1.06\%} &6.19\% &7.37\% &10.95\% &35.48\%\\ CNN-2 (ReLU) & 2.11\%&
  5.64\%& 8.27\%& 10.17\%& 32.43\%\\ CNN-2 (Quad) & 1.75\%& 5.30\%&
  8.83\%& 11.60\%& 36.90\%\\ CCNN-2 & 1.38\%& {\bf 4.32\%}& 6.98\% &
  {\bf 7.46\%}& {\bf 30.23\%}\\ \hline
\end{tabular}
\caption{Classification error on the basic MNIST and its four
    variations. The best performance within each block is bolded. The
    tag ``ReLU'' and ``Quad'' means ReLU activation and quadratic
    activation, respectively.}
\label{table:mnist-error} 
\end{table}

\begin{table}[t]
\small
\begin{tabular}{|l|c|}\hline
 & Classification error\\\hline
Full model CCNN-2 & 30.23\%\\\hline
Linear kernel & 38.47\%\\\hline
No ZCA whitening & 35.86\%\\\hline
Fewer random features & 31.95\%\\\hline
$\lfs{A} \leq R$ constraint & 30.66\%\\\hline
Early stopping & 30.63\% \\\hline
\end{tabular}
\caption{Factors that affect the training of CCNN-2: changing the kernel function, removing the data whitening or decreasing the number of random features has non-negligible impact to the performance. The results are reported on the {\tt img+rot} dataset.
  }
\label{table:CCNN-2-variants} 
\end{table}

Table~\ref{table:mnist-error} shows the classification errors on the
test set. The models are grouped with respect to the number of layers
that they contain. For models with one convolutional layer, the errors
of CNN-1 are significantly lower than that of NN-1, highlighting the
benefits of parameter sharing. The CCNN-1 model outperforms CNN-1 on
all datasets. For models with two or more hidden layers, the CCNN-2
model outperforms CNN-2 on all datasets, and is competitive against
the state-of-the-art.  In particular, it achieves the best accuracy on
the {\tt rand}, {\tt img} and {\tt img+rot} dataset, and is comparable
to the state-of-the-art on the remaining two datasets.

In order to understand the key factors that affect the training of
CCNN filters, we evaluate five variants: 
\begin{enumerate}
    \setlength{\itemsep}{0pt}
    \setlength{\parskip}{0pt}
    \setlength{\parsep}{0pt}
\item[(1)] replace the Gaussian kernel by a linear kernel; 
\item[(2)] remove the ZCA whitening in preprocessing; 
\item[(3)] use fewer random features ($m=200$ rather than $m=500$) to approximate the
  kernel matrix;
\item[(4)] regularize the parameter matrix by the Frobenius norm
  instead of the nuclear norm;
\item[(5)] stop the mini-batch SGD early before it converges.
\end{enumerate}
We evaluate the obtained filters by training a
second convolutional layer on top of them, then evaluating the
classification error on the hardest dataset {\tt img+rot}. As
Table~\ref{table:CCNN-2-variants} shows, switching to the linear
kernel or removing the ZCA whitening significantly degenerates the
performance. This is because that both variants equivalently modify
the kernel function. Decreasing the number of random features also
has a non-negligible effect, as it makes the kernel approximation
less accurate. These observations highlight the impact of the kernel function. Interestingly, replacing the nuclear norm by a Frobenius norm or stopping the algorithm early doesn't hurt the performance. To see their impact on the parameter matrix, we compute the effective rank (ratio between the nuclear norm and the spectral norm, see~\cite{eldar2012compressed}) of matrix $\Ahat$. The effective rank obtained by the last two variants are equal to 77 and 24, greater than that of the original CCNN (equal to 12). It reveals that the last two variants have damaged the algorithm's capability of enforcing a low-rank solution. However, the CCNN filters are retrieved from the top-$r$ singular vectors of the parameter matrix, hence the performance will remain stable as long as the top singular vectors are robust to the variation of the matrix.

In Section~\ref{sec:theory}, we showed that if the activation function
is a polynomial function, then the CCNN requires lower sample complexity to
match the performance of the best possible CNN. More precisely, if the activation function is degree-$\ell$ polynomial, then $C_\sigma(B)$ in the upper bound will be controlled by $\order(B^\ell)$.
This motivates us to study the performance of low-degree polynomial activations.
Table~\ref{table:mnist-error} shows that the CNN-2 model with a
quadratic activation function achieves error rates comparable to that
with a ReLU activation: CNN-2 (Quad) outperforms CNN-2 (ReLU) on the
{\tt basic} and {\tt rand} datasets, and is only slightly worse on the
{\tt rot} and {\tt img} dataset. Since the performance of CCNN matches that of the best possible CNN, the good performance of the quadratic activation in part explains why the CCNN is also good. 


\subsection{CIFAR-10}

In order to test the capability of CCNN in complex classification
tasks, we report its performance on the CIFAR-10
dataset~\cite{krizhevsky2009learning}.  The dataset consists of 60000
images divided into 10 classes. Each image has $32\times 32$ pixels in
RGB colors. We use 50k images for training and 10k images for testing.

\myparagraph{Implementation details} We train CNN and CCNN models with
two, three, and four layers
Each convolutional layer is
constructed on $5\times 5$ patches with unit stride, followed by
$3\times 3$ average pooling with two-pixel stride. We train 32, 32, 64
filters for the three convolutional layers from bottom to the top.  For
any $s\times s$ input, zero pixels are padded on its borders so that
the input size becomes $(s+4)\times (s+4)$, and the output size of the
convolutional layer is $(s/2)\times (s/2)$. The CNNs are activated by
the ReLU function. For CCNNs, we use the Gaussian kernel with hyperparameter
$\gamma = 1,2,2$ (for the three convolutional layers). The
feature matrix $Z(x)$ is constructed via random feature approximation
with dimension $m=2000$. The preprocessing steps are the same as in
the MNIST experiments. It was known that the generalization
performance of the CNN can be improved by training on random crops of the
original image~\cite{krizhevsky2012imagenet}, so we train the CNN on random
$24\times 24$ patches of the image, and test on the central $24\times
24$ patch. We also apply random cropping to training the the first and the second layer of the CCNN.

We compare the CCNN against other baseline methods that don't involve
nonconvex optimization: the kernel SVM with Fastfood Fourier features
(SVM$_{\rm Fastfood}$)~\cite{le2013fastfood}, the PCANet-2
model~\cite{chan2015pcanet} and the convolutional kernel networks
(CKN)~\cite{mairal2014convolutional}.

\begin{figure}
\begin{floatrow}
\hspace{-10pt} 
\capbtabbox{\small\qquad\qquad
  \begin{tabular}{|l|c|}\hline
 & Error rate \\\hline
CNN-1 & 34.14\%\\
CCNN-1 & {\bf 23.62\%}\\\hline
CNN-2 &  24.98\%\\
CCNN-2 & {\bf 20.52\%}\\\hline
SVM$_{\rm Fastfood}$~\cite{le2013fastfood} & 36.90\%\\
PCANet-2~\cite{chan2015pcanet} & 22.86\%\\
CKN~\cite{mairal2014convolutional} & 21.70\%\\
CNN-3 & 21.48\%\\
CCNN-3 & {\bf 19.56\%}\\
\hline
\end{tabular}\qquad\vspace{20pt}}
{\caption{Classification error on the CIFAR-10 dataset. The best
    performance within each block is bolded.}
\label{table:cifar10-error}} 
\ffigbox[0.45\textwidth]{\includegraphics[height =
    0.32\textwidth]{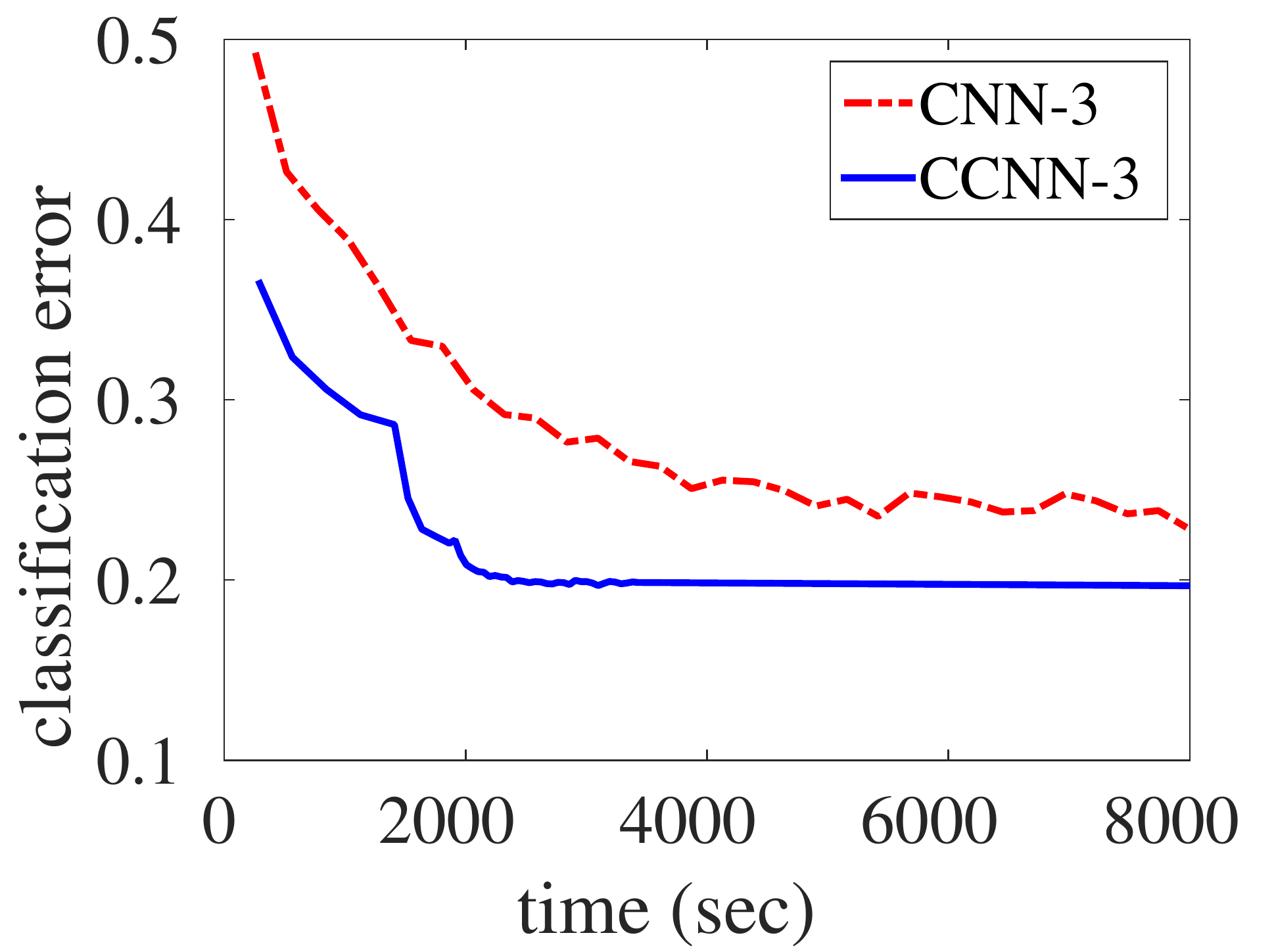}} {\caption{The convergence of
    CNN-3 and CCNN-3 on the CIFAR-10 dataset.}
  \label{fig:runtime-cifar10}}
\end{floatrow}
\end{figure}


\myparagraph{Results} 

We report the classification errors on in
Table~\ref{table:cifar10-error}.  For models of all depths, the CCNN
model outperforms the CNN model. The CCNN-2 and the CCNN-3 model also
outperform the three baseline methods. The advantage of the CCNN is
substantial for learning two-layer networks, when the optimality
guarantee of CCNN holds. The performance improves as more layers are
stacked, but as we observe in Table~\ref{table:cifar10-error}, the
marginal gain of CCNN diminishes as the network grows deeper. We suspect
that this is due to the greedy fashion in which the CCNN layers are constructed.
Once trained, the
low-level filters of the CCNN are no longer able to adapt to the final
classifier. In contrast, the low-level filters of the CNN model are continuously adjusted
via backpropagation.

It is worth noting that the performance of the CNN can be further improved
by adding more layers, switching from average pooling to max pooling,
and being regularized by local response normalization and
dropout~(see, e.g.~\cite{krizhevsky2012imagenet}).  The figures in
Table~\ref{table:cifar10-error} are by no means the state-of-the-art
result on CIFAR-10. However, it does demonstrate that the convex
relaxation is capable of improving the performance of convolutional
neural networks.  For future work, we propose to study a better way
for convexifying deep CNNs.

In Figure~\ref{fig:runtime-cifar10}, we compare the computational
efficiency of CNN-3 to its convexified version CCNN-3.  Both models
are trained by mini-batch SGD (with batchsize equal to 50) on a single
processor. We optimized the choice of step-size for each algorithm.
From the plot, it is easy to identify the three stages of the CCNN-3
curve for training the three convolutional layers from bottom to the
top.  We also observe that the CCNN converges faster than the CNN.  More
precisely, the CCNN takes half the runtime of the CNN to reach an error rate
of 28\%, and one-fifth of the runtime to reach an error rate of 23\%.
The per-iteration cost for training the first layer of CCNN is about
109\% of the per-iteration cost of CNN, but the per-iteration cost for
training the remaining two layers are about 18\% and 7\% of that of
CNN. Thus, training CCNN scales well to large datasets.

\begin{table}
\small
\begin{tabular}{|l|r|r|r|}
\hline
& CNN-1& CNN-2 & CNN-3\\\hline
Original & 34.14\% & 24.98\% & 21.48\% \\\hline
Convexified & {\bf 23.62\% }& {\bf 21.88\%}& {\bf 18.18\%}\\\hline
\end{tabular}
\caption{Comparing the original CNN and the one whose top convolution
  layer is convexified by CCNN. The classification errors are reported
  on CIFAR-10.}
\label{table:convexify-top-layer}
\end{table}


\myparagraph{Training a CCNN on top of a CNN} 

Instead of training a CCNN from scratch, we can also train CCNNs on top
of existing CNN layers. More concretely, once a CNN-$k$ model is
obtained, we train a two-layer CCNN by taking the $(k-1)$-th hidden
layer of CNN as input. This approach preserves the low-level features
learned by CNN, only convexifying its top convolutional layer.  The
underlying motivation is that the traditional CNN is good at learning
low-level features through backpropagation, while the CCNN is optimal
in learning two-layer networks.

In this experiment, we convexify the top convolutional layer of CNN-2
and CNN-3 using the CCNN approach, with a smaller Gaussian kernel
parameter (i.e.~$\gamma = 0.1$) and keeping other hyperparameters the
same as in the training of CCNN-2 and CCNN-3. The results are shown in
Table~\ref{table:convexify-top-layer}. The convexified CNN achieves
better accuracy on all network depths. It is worth noting that the
time for training a convexified layer is only a small fraction of the
time for training the original CNN.


\section{Related work}
\label{sec:related-work}

With the empirical success of deep neural networks, there has been an
increasing interest in theoretical understanding.
\citet{bengio2005convex} showed how to formulate neural network
training as a convex optimization problem involving an infinite number
of parameters. This perspective encourages incrementally adding
neurons to the network, whose generalization error was studied by
\citet{bach2014breaking}. \citet{zhang2015learning} propose a
polynomial-time ensemble method for learning fully-connected neural
networks, but their approach handles neither parameter sharing nor the
convolutional setting.  Other relevant works for learning
fully-connected networks
include~\cite{sedghi2014provable,janzamin2015generalization,
  livni2014computational}. Aslan et
al.~\cite{aslan2013convex,aslan2014convex} propose a method for
learning multi-layer latent variable models. They showed that for
certain activation functions, the proposed method is a convex
relaxation for learning the fully-connected neural network.

Another line of work is devoted to understanding the energy landscape of
a neural network. Under certain assumptions on the data distribution,
it can be shown that any local minimum of a two-layer fully connected
neural network has an objective value that is close to the global
minimum value~\cite{dauphin2014identifying,choromanska2014loss}. If
this property holds, then gradient descent can find a solution that is
``good enough''.  Similar results have also been established for
over-specified neural networks~\cite{safran2015quality}, or neural
networks that has a certain parallel
topology~\cite{haeffele2015global}.  However, these results are not
applicable to a CNN, since the underlying assumptions are not
satisfied by CNNs.

Past work has studied learning translation invariant features without
backpropagation. \citet{mairal2014convolutional} present convolutional kernel networks. They propose a translation-invariant kernel whose feature mapping can be approximated
by a composition of the convolution, non-linearity and pooling
operators, obtained through unsupervised learning.  However, this
method is not equipped with the optimality guarantees that we have
provided for CCNNs in this paper, even for learning one convolution
layer.  The ScatNet method~\cite{bruna2013invariant} uses translation
and deformation-invariant filters constructed by wavelet analysis;
however, these filters are independent of the data, unlike the
analysis in this paper. \citet{daniely2016toward} show that a randomly
initialized CNN can extract features as powerful as kernel methods,
but it is not clear how to provably improve the model from a random
initialization.


\section{Conclusion}
\label{sec:conclusion}

In this paper, we have shown how convex optimization can be used to
efficiently optimize CNNs as well as understand them
statistically. Our convex relaxation consists of two parts: the
nuclear norm relaxation for handling parameter sharing, and the RKHS
relaxation for handling non-linearity. For the two-layer CCNN, we
proved that its generalization error converges to that of the best
possible two-layer CNN. We handled multi-layer CCNNs only
heuristically, but observed that adding more layers improves the
performance in practice. On real data experiments, we demonstrated
that CCNN outperforms the traditional CNN of the same depth, is
computationally efficient, and can be combined with the traditional
CNN to achieve better performance. A major open problem is to formally
study the convex relaxation of deep CNNs.

\subsection*{Acknowledgements}

This work was partially supported by Office of Naval Research MURI
grant DOD-002888, Air Force Office of Scientific Research Grant
AFOSR-FA9550-14-1-001, Office of Naval Research grant ONR-N00014,
National Science Foundation Grant CIF-31712-23800,
as well as a Microsoft Faculty Research Award to the second author.

\appendix


\section{Inverse polynomial kernel and Gaussian kernel}
\label{sec:ipk-and-gaussian-kernel}

In this appendix, we describe the properties of the two types of kernels --- the inverse polynomial kernel~\eqref{eqn:def-ipk} and the Gaussian RBF kernel~\eqref{eqn:def-gaussian-kernel}. 
We prove that the associated reproducing kernel Hilbert Spaces (RKHS) of these kernels contain filters taking the form $h: z \mapsto \sigma(\langle w, z\rangle)$ for particular activation functions $\sigma$. 

\subsection{Inverse polynomial kernel}

We first verify that the function~\eqref{eqn:def-ipk} is a kernel function. This holds since that we can find a mapping $\varphi: \R^{d_1} \to \ell^2(\N)$ such that $\kf(z,z') = \langle \varphi(z),\varphi(z') \rangle$. We use $z_i$ to represent the $i$-th coordinate of an infinite-dimensional vector $z$. 
The $(k_1,\dots,k_j)$-th coordinate of $\varphi(z)$, where $j\in \N$ and $k_1,\dots,k_j\in [d_1]$, is defined as $2^{-\frac{j+1}{2}}x_{k_1}\dots x_{k_j}$. By this definition, we have
\begin{align}\label{eqn:psi-product}
	\langle \varphi(x), \varphi(y) \rangle &= \sum_{j=0}^\infty 2^{-(j+1)} \sum_{(k_1,\dots,k_j)\in [d_1]^j} z_{k_1}\dots z_{k_j} z'_{k_1}\dots z'_{k_j}.
\end{align}
Since $\ltwos{z}\leq 1$ and $\ltwos{z'} \leq 1$, the series on the right-hand side is absolutely convergent. 
The inner term on the right-hand side of equation~\eqref{eqn:psi-product} can be simplified to
\begin{align}\label{eqn:simplify-sum-product-terms}
\sum_{(k_1,\dots,k_j)\in [d_1]^j} z_{k_1}\dots z_{k_j} z'_{k_1}\dots z'_{k_j} = (\langle z, z' \rangle)^j.
\end{align}
Combining equations~\eqref{eqn:psi-product} and~\eqref{eqn:simplify-sum-product-terms} and using the fact that $|\langle z, z' \rangle| \leq 1$, we have
\begin{align*}
\langle \varphi(z), \varphi(z') \rangle &= \sum_{j=0}^\infty 2^{-(j+1)} (\langle z, z' \rangle)^j \stackrel{\rm (i)}{=} \frac{1}{2 - \langle z, z' \rangle} = \kf(z,z'),
\end{align*}
which verifies that $\kf$ is a kernel function and $\varphi$ is the associated feature map. 
Next, we prove that the associated RKHS contains the class of nonlinear filters. The lemma was proved by \citet{zhang2015ell_1}. We include the proof to make the paper self-contained.

\begin{lemma}\label{lemma:ipk-contains-filter}
Assume that the function $\sigma(x)$ has a polynomial expansion $\sigma(t) = \sum_{j=0}^\infty a_j t^j$. Let $C_\sigma(\lambda) \defeq \sqrt{\sum_{j=0}^\infty 2^{j+1}a_j^2\lambda^{2j}}$. If $C_\sigma(\ltwos{w}) < \infty$, then the RKHS induced by the inverse polynomial kernel contains function $h: z \mapsto \sigma(\inprod{w}{z})$ with Hilbert norm $\hnorms{h} = C_\sigma(\ltwos{w})$.
\end{lemma}

\begin{proof}
Let $\varphi$ be the feature map that we have defined for the polynomial inverse kernel. We define vector $\wdiamond\in \ell^2(\N)$ as follow: the $(k_1,\dots,k_j)$-th coordinate of $\wdiamond$, where $j\in \N$ and $k_1,\dots,k_j\in [d_1]$, is equal to $2^{\frac{j+1}{2}} a_j w_{k_1}\dots w_{k_j}$. By this definition, we have
\begin{align}\label{eqn:linearize-sigma}
 \sigma(\langle w, z \rangle)= \sum_{t=0}^\infty a_j (\langle w, z \rangle)^j =\sum_{j=0}^\infty a_j \sum_{(k_1,\dots,k_j)\in [d_1]^j} w_{k_1}\dots w_{k_j} z_{k_1}\dots z_{k_j} = 	\langle \wdiamond, \varphi(z) \rangle,
\end{align}
where the first equation holds since $\sigma(x)$ has a polynomial expansion $\sigma(x) = \sum_{j=0}^\infty a_j x^j$, the second by expanding the inner product, and the third by definition of $\wdiamond$ and $\varphi(z)$. The $\ell_2$-norm of $\wdiamond$ is equal to:
\begin{align}
\ltwos{\wdiamond}^2 &= \sum_{j=0}^\infty 2^{j+1} a_j^2 \sum_{(k_1,\dots,k_j)\in [d_1]^j}  \wdiamond_{k_1}^2 \wdiamond_{k_2}^2 \cdots \wdiamond_{k_j}^2= \sum_{j=0}^\infty 2^{j+1} a_j^2 \ltwos{\wdiamond}^{2j} = C_\sigma^2(\ltwos{\wdiamond}) < \infty.\label{eqn:gip-norm-induction}
\end{align}
By the basic property of the RKHS, the Hilbert norm of $h$ is equal to the $\ell_2$-norm of $\wdiamond$.
Combining equations~\eqref{eqn:linearize-sigma} and~\eqref{eqn:gip-norm-induction}, we conclude that $h\in \Hset$ and
$\hnorms{h} = \ltwos{\wdiamond} = C_\sigma(\ltwos{\wdiamond})$.
\end{proof}

According to Lemma~\ref{lemma:ipk-contains-filter}, it suffices to upper bound $C_\sigma(\lambda)$ for a particular activation function $\sigma$. To make $C_\sigma(\lambda) < \infty$, the coefficients $\{a_j\}_{j=0}^\infty$ must quickly converge to zero, meaning that the activation function must be sufficiently smooth.
For polynomial functions of degree $\ell$, the definition of $C_\sigma$ implies that $C_\sigma(\lambda) = \order(\lambda^\ell)$. For the sinusoid activation $\sigma(t) \defeq \sin(t)$, we have
\[
 C_\sigma(\lambda) = \sqrt{\sum_{j=0}^\infty \frac{2^{2j+2}}{((2j+1)!)^2}\cdot (\lambda^2)^{2j+1}} \leq 2 e^{\lambda^2}.
\]
For the erf function and the smoothed hinge loss function defined in Section~\ref{sec:theory}, \citet[][Proposition 1]{zhang2015ell_1} proved that $C_\sigma(\lambda) = \order(e^{c\lambda^2})$ for universal numerical constant $c > 0$.

\subsection{Gaussian kernel}

The Gaussian kernel also induces an RKHS that contains a particular class of nonlinear filters. The proof is similar to that of Lemma~\ref{lemma:ipk-contains-filter}.

\begin{lemma}\label{lemma:gaussian-contains-filter}
Assume that the function $\sigma(x)$ has a polynomial expansion $\sigma(t) = \sum_{j=0}^\infty a_j t^j$. Let $C_\sigma(\lambda) \defeq \sqrt{\sum_{j=0}^\infty \frac{j!e^{2\gamma}}{(2\gamma)^j} a_j^2\lambda^{2j}}$. If $C_\sigma(\ltwos{w}) < \infty$, then the RKHS induced by the Gaussian kernel contains the function $h: z \mapsto \sigma(\inprod{w}{z})$ with Hilbert norm $\hnorms{h} = C_\sigma(\ltwos{w})$.
\end{lemma}

\begin{proof}
When $\ltwos{z} = \ltwos{z'} = 1$, It is well-known~\cite[see, e.g.][]{SteChr08} the following mapping $\varphi: \R^{d_1} \to \ell^2(\N)$ is a feature map for the Gaussian RBF kernel: the $(k_1,\dots,k_j)$-th coordinate of $\varphi(z)$, where $j\in \N$ and $k_1,\dots,k_j\in [d_1]$, is defined as $e^{-\gamma}((2\gamma)^j/j!)^{1/2} x_{k_1}\dots x_{k_j}$.
Similar to equation~\eqref{eqn:linearize-sigma}, we define a vector $\wdiamond\in \ell^2(\N)$ as follow: the $(k_1,\dots,k_j)$-th coordinate of $\wdiamond$, where $j\in \N$ and $k_1,\dots,k_j\in [d_1]$, is equal to $e^\gamma ((2\gamma)^j/j!)^{-1/2} a_j w_{k_1}\dots w_{k_j}$. By this definition, we have
\begin{align}\label{eqn:linearize-sigma-for-rbf}
 \sigma(\langle w, z \rangle)= \sum_{t=0}^\infty a_j (\langle w, z \rangle)^j =\sum_{j=0}^\infty a_j \sum_{(k_1,\dots,k_j)\in [d_1]^j} w_{k_1}\dots w_{k_j} z_{k_1}\dots z_{k_j} = 	\langle \wdiamond, \varphi(z) \rangle.
\end{align}
The $\ell_2$-norm of $\wdiamond$ is equal to:
\begin{align}
\ltwos{\wdiamond}^2 &= \sum_{j=0}^\infty \frac{j!e^{2\gamma}}{(2\gamma)^j} a_j^2 \sum_{(k_1,\dots,k_j)\in [d_1]^j}  \wdiamond_{k_1}^2 \wdiamond_{k_2}^2 \cdots \wdiamond_{k_j}^2= \sum_{j=0}^\infty \frac{j!e^{2\gamma}}{(2\gamma)^j} a_j^2 \ltwos{\wdiamond}^{2j} = C_\sigma^2(\ltwos{\wdiamond}) < \infty.\label{eqn:gip-norm-induction-for-rbf}
\end{align}
Combining equations~\eqref{eqn:linearize-sigma} and~\eqref{eqn:gip-norm-induction}, we conclude that $h\in \Hset$ and $\hnorms{h} = \ltwos{\wdiamond} = C_\sigma(\ltwos{\wdiamond})$.
\end{proof}

Comparing Lemma~\ref{lemma:ipk-contains-filter} and Lemma~\ref{lemma:gaussian-contains-filter}, we find that the Gaussian kernel imposes a stronger condition on the smoothness of the activation function.
 For polynomial functions of degree $\ell$, we still have $C_\sigma(\lambda) = \order(\lambda^\ell)$. For the sinusoid activation $\sigma(t) \defeq \sin(t)$, it can be verified that
\[
 C_\sigma(\lambda) = \sqrt{e^{2\gamma} \sum_{j=0}^\infty \frac{1}{(2j+1)!}\cdot \Big(\frac{\lambda^2}{2\gamma}\Big)^{2j+1}} \leq e^{\lambda^2/(4\gamma)+\gamma}.
\]
However, the value of $C_\sigma(\lambda)$ is infinite when $\sigma$ is the erf function or the smoothed hinge loss, meaning that the Gaussian kernel's RKHS doesn't contain filters activated by these two functions.

\section{Convex relaxation for nonlinear activation}
\label{sec:nonlinear-activation}

In this appendix, we provide a detailed derivation of the relaxation
for nonlinear activation functions that we previously sketched in
Section~\ref{sec:nonlinear}. Recall that the filter output is $\sigma(
\inprod{w_j}{z})$. Appendix~\ref{sec:ipk-and-gaussian-kernel} shows that given a sufficiently smooth activation function
$\sigma$,  we can find some kernel function
$\kf: \R^{d_1}\times \R^{d_1}\to \R$ and a feature map $\varphi:\R^{d_1}\to \ell^2(\N)$ satisfying $\kf(z,z') \equiv \inprod{\varphi(z)}{\varphi(z')}$, such that
\begin{align}
\label{eqn:nonlinear-to-linear}
\sigma(\inprod{w_j}{z}) \equiv \inprod{\wdiamond_j}{\varphi(z)}.
\end{align}
Here $\wdiamond_j \in \ell^2(\N)$ is a countable-dimensional vector and $\varphi \defeq (\varphi_1,\varphi_2,\dots)$ is a countable sequence of functions. 
Moreover, the $\ell_2$-norm of $\wdiamond_j$ is bounded as $\ltwos{\wdiamond_j}\leq
C_\sigma(\ltwos{w_j})$ for a monotonically increasing
function~$C_\sigma$ that depends on the kernel (see Lemma~\ref{lemma:ipk-contains-filter} and Lemma~\ref{lemma:gaussian-contains-filter}). As a
consequence, we may use $\varphi(z)$ as the vectorized representation
of the patch $z$, and use $\wdiamond_j$ as the linear transformation
weights, then the problem is reduced to training a CNN with the identity activation function. 

The filter is parametrized by an infinite-dimensional vector $\wdiamond_j$.
Our next step is to reduce the original ERM problem to a finite-dimensional one. In order to minimize the empirical risk, one only needs to concern the output  on the training data, that is, the output of $\inprod{\wdiamond_j}{\varphi(z_p(x_i))}$ for all $(i,p)\in [n]\times[P]$. Let $T$ be the orthogonal projector onto the linear subspace spanned by the vectors $\{\varphi(z_p(x_i)): (i,p)\in [n]\times [P]\}$. Then we have
\[
	\forall~ (i,p)\in [n]\times [P]:\quad \inprod{\wdiamond_j}{\varphi(z_p(x_i))} = \inprod{\wdiamond_j}{T \varphi(z_p(x_i))} = \inprod{T \wdiamond_j}{\varphi(z_p(x_i))}.
\]
The last equation follows since the orthogonal projector $T$ is self-adjoint. Thus, for empirical risk minimization, we can without loss of generality assume that $\wdiamond_j$ belongs to the linear subspace spanned by $\{\varphi(z_p(x_i)): (i,p)\in [n]\times [P]\}$ and 
reparametrize it by:
\begin{align}
\label{eqn:representer-equation}
\wdiamond_j = \sum_{(i,p)\in [n]\times [P]} \beta_{j,(i,p)}
\varphi(z_p(x_i)).
\end{align}
Let $\beta_j \in \R^{n P}$ be a vector whose whose $(i,p)$-th
coordinate is $\beta_{j,(i,p)}$. In order to estimate $\wdiamond_j$,
it suffices to estimate the vector $\beta_j$. By definition, the
vector satisfies the relation $\beta_j^\top K \beta_j =
\ltwos{\wdiamond_j}^2$, where $K$ is the $nP \times nP$ kernel matrix
defined in Section~\ref{sec:nonlinear}. As a consequence, if we can
find a matrix $Q$ such that $Q Q^\top = K$, then we have the norm
constraint
\begin{align}\label{eqn:evaluate-hilbert-norm}
\ltwos{Q^\top \beta_j} = \sqrt{\beta_j^\top K \beta_j} =
\ltwos{\wdiamond_j} \leq C_\sigma(\ltwos{w_j}) \leq C_\sigma(B).
\end{align}
Let $v(z) \in \real^{n P}$ be a vector whose $(i,p)$-th coordinate is
equal to $\kf(z, z_p(x_i))$. Then by
equations~\eqref{eqn:nonlinear-to-linear}
and~\eqref{eqn:representer-equation}, the filter output can be written
as
\begin{align}\label{eqn:output-can-be-written-as}
\sigma \big( \inprod{w_j}{z} \big) & \equiv
\inprod{\wdiamond_j}{\varphi(z)} \, \equiv \, \inprod{\beta_j}{v(z)}.
\end{align}
For any patch $z_p(x_i)$ in the training data, the vector $v(z_p(x_i))$ belongs to the column space of
the kernel matrix $K$. Therefore, letting $Q^\dagger$ represent the
pseudo-inverse of \mbox{matrix $Q$,} we have
\begin{align*}
\forall~ (i,p)\in [n]\times [P]:\quad \inprod{ \beta_j}{v(z_p(x_i))} = \beta_j^\top Q Q^\dagger v(z_p(x_i)) = \inprod{(Q^\top)^\dagger Q^\top\beta_j}{v(z_p(x_i))}.
\end{align*}
It means that if we replace the vector $\beta_j$ on the right-hand side of equation~\eqref{eqn:output-can-be-written-as} by the vector $(Q^\top)^\dagger Q^\top \beta_j$, then it won't change the empirical risk. 
Thus, for ERM we can parametrize the filters by
\begin{align}\label{eqn:final-representation-of-filters}
h_j(z) \defeq
\inprod{(Q^\top)^\dagger Q^\top\beta_j}{v(z)} = \inprod{Q^\dagger v(z)}{Q^\top\beta_j}.
\end{align}

Let $Z(x)$ be an $P \times nP$ matrix whose $p$-th row is equal to
$Q^\dagger v(z_p(x))$. Similar to the steps in
equation~\eqref{eqn:simple-model-output}, we have
\begin{align*}
f_\outindex(x) = \sum_{j=1}^r \alpha_{\outindex,j}^\top Z(x)
K^{1/2}\beta_j = \tr\Big(Z(x) \Big(\sum_{j=1}^r
K^{1/2}\beta_j\alpha_{\outindex,j}^\top \Big)\Big) =
\tr(Z(x)A_\outindex),
\end{align*}
where $A_\outindex \defeq \sum_{j=1}^r
Q^\top\beta_j\alpha_{\outindex,j}^\top$.  If we let $A \defeq
(A_1,\dots,A_{d_2})$ denote the concatenation of these matrices, then
this larger matrix satisfies the constraints:
\begin{description}
\item[Constraint (C1):] $\max \limits_{j \in [r]} \ltwos{Q^\top
  \beta_j} \leq C_\sigma(\bou_1)$ and $\max \limits_{(\outindex, j)
  \in [d_2] \times [r]} \ltwos{\alpha_{\outindex,j}} \leq \bou_2$.
\item[Constraint (C2):] The matrix $A$ has rank at most $r$.
\end{description}
We relax these two constraints to the nuclear norm constraint:
\begin{align}
\label{eqn:trace-constraint}
\lncs{A} \leq C_\sigma(B_1) B_2 r \sqrt{d_2}.
\end{align}
By comparing constraints~\eqref{eqn:simple-model-nuclear-constraint}
and~\eqref{eqn:trace-constraint}, we see that the only difference is
that the term $B_1$ in the norm bound has been replaced by
$C_\sigma(B_1)$.  This change is needed because we have used the
kernel trick to handle nonlinear activation functions.


\section{Proof of Theorem~\ref{theorem:single-layer-rcnn}}
\label{sec:proof-main-theorem}

Since the output is one-dimensional in this case, we can adopt the
simplified notation $(A,\alpha_{j,p})$ for the matrix $(A_1,
\alpha_{1,j,p})$.  Letting $\Hset$ be the RKHS associated with the
kernel function $\kf$, and letting
$\norms{\cdot}_{\Hset}$ be the associated Hilbert norm, consider the
function class
\begin{align}
\label{eqn:rcnn-class}
\rcnn \defeq \Big\{ x\mapsto \sum_{j=1}^{r^*} \sum_{p=1}^{P} \alpha_{j,p}
h_j(z_p(x)) :\; r^* < \infty \mbox{ and } \sum_{j=1}^{r^*}
\ltwos{\alpha_j} \norms{h_j}_{\Hset} \leq C_\sigma(B_1) B_2 d_2
\Big\}.
\end{align}
Here $\alpha_{j,p}$ denotes the $p$-th entry of vector $\alpha_j\in
\R^P$, whereas the quantity $C_\sigma(B_1)$ only depends on $B_1$ and
the activation function~$\sigma$. The following lemma shows that the
function class $\rcnn$ is rich enough so that it contains family of
CNN predictors as a subset.  The reader should recall the notion of a
\emph{valid activation function}, as defined prior to the statement of
Theorem~\ref{theorem:single-layer-rcnn}.
\begin{lemma}
\label{lemma:rcnn-contains-cnn}
For any valid activation function $\sigma$, there is a quantity
$C_\sigma(B_1)$, depending only on $B_1$ and $\sigma$, such that $\cnn
\subset \rcnn$.
\end{lemma}
\noindent See Appendix~\ref{AppLemContain} for the proof. \\

Next, we connect the function class $\rcnn$ to the CCNN algorithm.
Recall that $\frcnn$ is the predictor trained by the CCNN
algorithm. The following lemma shows that $\frcnn$ is an empirical
risk minimizer within $\rcnn$.

\begin{lemma}
\label{lemma:erm}
With the CCNN hyper-parameter $R = C_\sigma(B_1) B_2 d_2$, the
predictor $\frcnn$ is guaranteed to satisfy the inclusion
\begin{align*}
\frcnn \in \arg\min_{f\in \rcnn} \sum_{i=1}^n \Loss(f(x_i);y_i).
\end{align*}
\end{lemma}
\noindent See Appendix~\ref{sec:proof-erm} for the proof. \\

Our third lemma shows that the function class $\rcnn$ is not ``too
big'', which we do by upper bounding its Rademacher complexity.  The
Rademacher complexity of a function class \mbox{$\mathcal{F} = \{f:
  \mathcal{X} \to \R \}$} with respect to $n$ i.i.d.~samples
$\{X_i\}_{i=1}^n$ is given by
\begin{align*}
\PopRade(\mathcal{F}) \defeq \E_{X, \epsilon} \left[ \sup_{f\in
    \mathcal{F}} \frac{1}{n} \sum_{i=1}^n \epsilon_i f(X_i) \right],
\end{align*}
 where $\{\epsilon_i\}_{i=1}^n$ are
an i.i.d. sequence of uniform $\{-1, +1\}$-valued variables.
Rademacher complexity plays an important role in empirical process
theory, and in particular can be used to bound the generalization loss
of our empirical risk minimization problem.  We refer the reader
to~\citet{bartlett2003rademacher} for an introduction to the
theoretical properties of Rademacher complexity.

The following lemma involves the kernel matrix $K(x)\in \R^{P\times
  P}$ whose $(i,j)$-th entry is equal to $\kf(z_i(x), z_j(x))$, as
well as the expectation $\E[\ltwos{K(X)}]$ of the spectral norm of
this matrix when $X$ is drawn randomly.

\begin{lemma}
\label{lemma:rademacher}
There is a universal constant $c$ such that
\begin{align}
\label{EqnRademacherBound}
\PopRade(\rcnn) \leq \frac{c\; C_\sigma(B_1)B_2 r \sqrt{\log(n P)
    \E[\ltwos{K(X)}]}}{\sqrt{n} }.
\end{align}
\end{lemma}
\noindent See Appendix~\ref{sec:proof-rademacher} for the proof of
this claim. \\

Combining Lemmas~\ref{lemma:rcnn-contains-cnn}
through~\ref{lemma:rademacher} allows us to compare the CCNN predictor
$\frcnn$ against the best model in the CNN class.
Lemma~\ref{lemma:erm} shows that $\frcnn$ is the empirical risk
minimizer within function class $\rcnn$. Thus, the theory of
Rademacher complexity~\cite{bartlett2003rademacher} guarantees that
\begin{align}
\label{EqnChickenSoup}
\E[\Loss(\rcnn(X); Y)] \leq \inf_{f\in \rcnn} \E[\Loss(f(x); y)] +
2L\cdot \PopRade(\rcnn) + \frac{c}{\sqrt{n}},
\end{align}
 where $c$ is a universal constant.  By
Lemma~\ref{lemma:rcnn-contains-cnn}, we have
\begin{align*}
\inf_{f\in \rcnn} \E[\Loss(f(X); Y)] \leq \inf_{f\in \cnn}
\E[\Loss(f(X); Y)].
\end{align*}
Plugging this upper bound into inequality~\eqref{EqnChickenSoup} and
applying Lemma~\ref{lemma:rademacher} completes the proof.


\subsection{Proof of Lemma~\ref{lemma:rcnn-contains-cnn}}
\label{AppLemContain}

With the activation functions specified in the lemma statement,
Lemma~\ref{lemma:ipk-contains-filter} and Lemma~\ref{lemma:gaussian-contains-filter} show that there is a quantity
$C_\sigma(B_1)$, such any filter of CNN belongs to the reproducing
kernel Hilbert space $\Hset$ and its Hilbert norm is bounded by
$C_\sigma(B_1)$. As a consequence, any function $f\in \cnn$ can be
represented by
\begin{align*}
f(x) \defeq \sum_{j=1}^{r} \sum_{p=1}^{P} \alpha_{j,p} h_j(z_p(x))
\quad \mbox{where $\norms{h_j}_{\Hset}\leq C_\sigma(B_1)$ and
  $\ltwos{\alpha_j}\leq B_2$}.
\end{align*}
It is straightforward to verify that function $f$ satisfies the
constraint in equation~\eqref{eqn:rcnn-class}, and consequently
belongs to $\rcnn$.


\subsection{Proof of Lemma~\ref{lemma:erm}}
\label{sec:proof-erm}

Let $\mathcal{C}_R$ denote the function class $\big\{ x \mapsto
\tr(Z(x)A) :\; \lncs{A}\leq R \big \}$. We first prove that
$\mathcal{C}_R \subset \rcnn$. Consider an arbitrary function $f_A(x)
\defn \tr(Z(x)A)$ belonging to $\mathcal{C}_R$.  Note that the matrix
$A$ has a singular value decomposition (SVD) $A = \sum_{j=1}^{r^*}
\lambda_j w_j u_j^\top$ for some ${r^*} < \infty$, where $w_j$ and
$u_j$ are unit vectors and $\lambda_j$ are real numbers. Using this
notation, the function $f_A$ can be represented as the sum
\begin{align*}
f_A(x) = \sum_{j=1}^{r^*} \lambda_j u_{j}^\top Z(x) w_j.
\end{align*}
Let $v(z)$ be an $n P$-dimensional vector whose $(i,p)$-th coordinate
is equal to $\kf(z, z_p(x_i))$. Then $Q^\dagger v(z_p(x))$ is the
$p$-th row of matrix $Z(x)$. Letting $h_j$ denote the mapping $z
\mapsto \langle Q^\dagger v(z), w_j\rangle$, we have
\begin{align*}
f_A(x) = \sum_{j=1}^{r^*} \sum_{p=1}^{P} \lambda_j u_{j,p}
h_j(z_p(x)).
\end{align*}
The function $h_j$ can also be written as $z
\mapsto \langle (Q^\top)^\dagger w_j, v(z)\rangle$. Equation~\eqref{eqn:evaluate-hilbert-norm} implies that the Hilbert norm of this function is equal to 
$\norms{Q^\top (Q^\top)^\dagger w_j}_2$, which is bounded by $\ltwos{w_j} = 1$. 
Thus we have
\begin{align*}
  \sum_{j=1}^{r^*} \ltwos{\lambda_j u_j} \norms{h_j}_{\Hset} =
  \sum_{j=1}^{r^*} |\lambda_j| = \lncs{A} \leq C_\sigma(B_1) B_2
  r,
\end{align*}
which implies that $f_A \in \rcnn$.

Next, it suffices to prove that for some empirical risk minimizer $f$
in function class $\rcnn$, it also belongs to the function class
$\mathcal{C}_R$. Recall that any function $f\in \rcnn$ can be
represented in the form
\begin{align*}
f(x) = \sum_{j=1}^{r^*} \sum_{p=1}^{P} \alpha_{j,p} h_j(z_p(x)).
\end{align*}
where the filter $h_j$ belongs to the RKHS. Let $\varphi: \R^{d_1}\to \ell^2(\N)$ be a feature map of the kernel function. The function $h_j(z)$ can be represented by $\inprod{\wdiamond}{\varphi(z)}$ for some $\wdiamond\in \ell^2(\N)$. 
In Appendix~\ref{sec:nonlinear-activation}, we have shown that any function taking this form can be replaced by $\inprod{(Q^\top)^\dagger Q^\top\beta_j}{v(z)}$ for some vector $\beta_j\in \R^{n P}$ without changing the output on the training data. Thus, there exists at least one empirical risk minimizer $f$ of $\rcnn$ such that all of its filters take the form 
\begin{align}\label{eqn:all-filters-take-the-form}
	h_j(z) = \inprod{(Q^\top)^\dagger Q^\top\beta_j}{v(z)}. 
\end{align}
By equation~\eqref{eqn:evaluate-hilbert-norm}, the Hilbert norm of these filters satisfy:
\begin{align*}
\norms{h_j}_{\Hset} = \ltwos{Q^\top (Q^\top)^\dagger Q^\top
  \beta_j} = \ltwos{Q^\top \beta_j}.
\end{align*}
According to Appendix~\ref{sec:nonlinear-activation}, if all filters take the form~\eqref{eqn:all-filters-take-the-form}, then the function $f$ can be represented by $\tr(Z(x)A)$ for matrix $A
\defeq \sum_{j=1}^{r^*} Q^\top \beta_j \alpha_j^\top$. Consequently,
the nuclear norm is bounded as
\begin{align*}
\lncs{A} \leq \sum_{j=1}^{r^*} \ltwos{\alpha_j}\ltwos{Q^\top \beta_j}
= \sum_{j=1}^{r^*} \ltwos{\alpha_j}\norms{h_j}_{\Hset} \leq
C_\sigma(B_1) B_2 r = R,
\end{align*}
which establishes that the function $f$ belongs to the function
class $\mathcal{C}_R$.


\subsection{Proof of Lemma~\ref{lemma:rademacher}}
\label{sec:proof-rademacher}

Throughout this proof, we use the shorthand notation $R \defeq
C_\sigma(B_1)B_2 r$. Recall that any function $f\in \rcnn$ can be
represented in the form
\begin{align}\label{eqn:rademacher-proof-function-expansion}
f(x) = \sum_{j=1}^{r^*} \sum_{p=1}^{P} \alpha_{j,p} h_j(z_p(x)) \quad
\mbox{where $h_j\in \Hset$},
\end{align}
and the Hilbert space $\Hset$ is induced by the kernel function
$\kf$. Since any patch $z$ belongs to the compact space $\{z:
\ltwos{z}\leq 1\}$ and $\kf$ is a continuous kernel satisfying
$\kf(z,z) \leq 1$, Mercer's theorem~\cite[][Theorem 4.49]{SteChr08} implies that there is a feature map $\varphi: \R^{d_1}\to \ell^2(\N)$ such that $\sum_{\ell = 1}^\infty \varphi_\ell(z)\varphi_\ell(z')$ converges uniformly and absolutely to $\kf(z,z')$.
Thus, we can write $\kf(z,z') = \inprod{\varphi(z)}{\varphi(z')}$. Since $\varphi$ 
is a feature map, every any function $h\in \Hset$ can be written as $h(z) = \inprod{\beta}{\varphi(z)}$ for some $\beta\in \ell^2(\mathbb{N})$, and the Hilbert norm of $h$ is equal to $\ltwos{\beta}$.

Using this notation, we can write the filter $h_j$ in
equation~\eqref{eqn:rademacher-proof-function-expansion} as $h_j(z) =
\langle \beta_j, \varphi(z)\rangle$ for some vector $\beta_j \in
\ell^2(\mathbb{N})$, with Hilbert norm $\hnorms{h_j} = \ltwos{\beta_j}$.  For each $x$, let $\Psi(x)$ denote the linear
operator that maps any sequence $\theta \in \ell^2(\mathbb{N})$ to the
vector in $\real^P$ with elements
\begin{align*}
\begin{bmatrix} \inprod{\theta}{\varphi(z_1(x))} & \ldots & \inprod{\theta}{\varphi(z_P(x))} 
\end{bmatrix}^T.
\end{align*}
Informally, we can think of $\Psi(x)$ as a matrix whose $p$-th row is
equal to $\varphi(z_p(x))$.  The function $f$ can then be written as
\begin{align}
\label{eqn:fx-expr}
f(x_i) = \sum_{j=1}^{{r^*}} \alpha_j^\top \Psi(x_i) \beta_j = \tr
\left(\Psi(x_i)\Big(\sum_{j=1}^{{r^*}} \beta_j \alpha_j^\top
\Big)\right).
\end{align}
The matrix $\sum_{j=1}^{{r^*}}
\beta_j \alpha_j^\top$ satisfies the constraint
\begin{align}
\label{eqn:fx-matrix-constraint}
\norm{\sum_{j=1}^{{r^*}} \beta_j \alpha_j^\top}_* \leq
\sum_{j=1}^{r^*} \ltwos{\alpha_j}\cdot \ltwos{\beta_j} =
\sum_{j=1}^{r^*} \ltwos{\alpha_j}\cdot \norms{h_j}_{\Hset} \leq R.
\end{align}
Combining equation~\eqref{eqn:fx-expr} and
inequality~\eqref{eqn:fx-matrix-constraint}, we find that the
Rademacher complexity is bounded by
\begin{align}
\label{eqn:rad-expand}
R_{n}(\rcnn) = \frac{1}{n}\E\left[\sup_{f\in \rcnn} \sum_{i=1}^{n}
  \epsilon_i f(x_i)\right] & \leq \frac{1}{n}\E\left[\sup_{A:
    \lncs{A}\leq R} \tr\left( \Big(\sum_{i=1}^{n} \epsilon_i
  \Psi(x_i)\Big) A\right) \right]\nonumber\\ 
&= \frac{R}{n} \E\left[ \norm{\sum_{i=1}^{n} \epsilon_i \Psi(x_i)}_2
  \right],
\end{align}
where the last equality uses H\"{o}lder's inequality---that is, the
duality between the nuclear norm and the spectral norm.

As noted previously, we may think informally of the quantity
$\sum_{i=1}^{n} \epsilon_i \Psi(x_i)$ as a matrix with $P$ rows and
infinitely many columns. Let $\Psi^{(d)}(x_i)$ denote the submatrix
consisting of the first $d$ columns of $\Psi(x_i)$ and let
$\Psi^{(-d)}(x_i)$ denote the remaining sub-matrix. We have
\begin{align*}
\E\left[ \norm{\sum_{i=1}^{n} \epsilon_i \Psi(x_i)}_2\right] &\leq
\E\left[\norm{\sum_{i=1}^{n} \epsilon_i \Psi^{(d)}(x_i)}_2\right] +
\left(\E\left[\norm{\sum_{i=1}^{n} \epsilon_i
    \Psi^{(-d)}(x_i)}_F^2\right]\right)^{1/2}\\ &\leq
\E\left[\norm{\sum_{i=1}^{n} \epsilon_i \Psi^{(d)}(x_i)}_2\right] +
\left(n P\cdot \E\left[\sum_{\ell=d+1}^\infty
  \varphi^2_{\ell}(z)\right]\right)^{1/2}.
\end{align*}
Since $\sum_{\ell=1}^\infty \varphi^2_{\ell}(z)$ uniformly converges to $\kf(z,z)$,
the second term on the right-hand side converges to zero as $d\to
\infty$. Thus it suffices to bound the first term and take the
limit.  In order to upper bound the spectral norm
$\norm{\sum_{i=1}^{n} \epsilon_i \Psi^{(d)}(x_i)}_2$, we use a matrix
Bernstein inequality due to~\citet[][Theorem 2.1]{minsker2011some}. In
particular, whenever $\tr(\Psi^{(d)}(x_i)(\Psi^{(d)}(x_i))^\top) \leq C_1$,
there is a universal constant $c$ such that the expected spectral
norm is upper bounded as
\begin{align*}
\E \left[\norm{\sum_{i=1}^{n} \epsilon_i \Psi^{(d)}(x_i)}_2\right] & \leq c
\; \sqrt{\log(n C_1)} \E\Big[\Big(\sum_{i=1}^n
  \ltwos{\Psi^{(d)}(x_i)(\Psi^{(d)}(x_i))^\top}\Big)^{1/2}\Big]\\
& \leq c \; \sqrt{\log(n C_1)} \Big( n
\E[\ltwos{\Psi(X)\Psi^\top(X)}]\Big)^{1/2}.
\end{align*}
Note that the uniform kernel expansion $\kf(z,z') =
  \sum_{\ell = 1}^\infty \varphi_\ell(z)\varphi_\ell(z')$ implies the trace norm bound
\mbox{$\tr(\Psi^{(d)}(x_i)(\Psi^{(d)}(x_i))^\top) \leq \tr(K(x_i))$.}
Since all patches are contained in the unit $\ell_2$-ball, the kernel
function $\kf$ is uniformly bounded by $1$, and hence $C_1\leq
P$. Taking the limit $d\to\infty$, we find that
\begin{align*}
\E \left[\norm{\sum_{i=1}^{n} \epsilon_i \Psi(x_i)}_2\right] \leq c \;
\sqrt{\log(n P)} \Big( n \E[\ltwos{K(X)}]\Big)^{1/2}.
\end{align*}
Finally, substituting this upper bound into
inequality~\eqref{eqn:rad-expand} yields the claimed
bound~\eqref{EqnRademacherBound}.


\bibliographystyle{abbrvnat}
  \bibliography{bib}

\begin{thebibliography}{48}
\providecommand{\natexlab}[1]{#1}
\providecommand{\url}[1]{\texttt{#1}}
\expandafter\ifx\csname urlstyle\endcsname\relax
  \providecommand{\doi}[1]{doi: #1}\else
  \providecommand{\doi}{doi: \begingroup \urlstyle{rm}\Url}\fi

\bibitem[Aslan et~al.(2013)Aslan, Cheng, Zhang, and
  Schuurmans]{aslan2013convex}
{\"O}.~Aslan, H.~Cheng, X.~Zhang, and D.~Schuurmans.
\newblock Convex two-layer modeling.
\newblock In \emph{Advances in Neural Information Processing Systems}, pages
  2985--2993, 2013.

\bibitem[Aslan et~al.(2014)Aslan, Zhang, and Schuurmans]{aslan2014convex}
{\"O}.~Aslan, X.~Zhang, and D.~Schuurmans.
\newblock Convex deep learning via normalized kernels.
\newblock In \emph{Advances in Neural Information Processing Systems}, pages
  3275--3283, 2014.

\bibitem[Bach(2014)]{bach2014breaking}
F.~Bach.
\newblock Breaking the curse of dimensionality with convex neural networks.
\newblock \emph{arXiv preprint arXiv:1412.8690}, 2014.

\bibitem[Bartlett and Mendelson(2003)]{bartlett2003rademacher}
P.~L. Bartlett and S.~Mendelson.
\newblock Rademacher and {G}aussian complexities: Risk bounds and structural
  results.
\newblock \emph{The Journal of Machine Learning Research}, 3:\penalty0
  463--482, 2003.

\bibitem[Bengio et~al.(2005)Bengio, Roux, Vincent, Delalleau, and
  Marcotte]{bengio2005convex}
Y.~Bengio, N.~L. Roux, P.~Vincent, O.~Delalleau, and P.~Marcotte.
\newblock Convex neural networks.
\newblock In \emph{Advances in Neural Information Processing Systems}, pages
  123--130, 2005.

\bibitem[Blum and Rivest(1992)]{blum1992training}
A.~L. Blum and R.~L. Rivest.
\newblock Training a 3-node neural network is {NP}-complete.
\newblock \emph{Neural Networks}, 5\penalty0 (1):\penalty0 117--127, 1992.

\bibitem[Bottou(1998)]{bottou1998online}
L.~Bottou.
\newblock Online learning and stochastic approximations.
\newblock \emph{On-line learning in neural networks}, 17\penalty0 (9):\penalty0
  142, 1998.

\bibitem[Bruna and Mallat(2013)]{bruna2013invariant}
J.~Bruna and S.~Mallat.
\newblock Invariant scattering convolution networks.
\newblock \emph{Pattern Analysis and Machine Intelligence, IEEE Transactions
  on}, 35\penalty0 (8):\penalty0 1872--1886, 2013.

\bibitem[Chan et~al.(2015)Chan, Jia, Gao, Lu, Zeng, and Ma]{chan2015pcanet}
T.-H. Chan, K.~Jia, S.~Gao, J.~Lu, Z.~Zeng, and Y.~Ma.
\newblock Pcanet: A simple deep learning baseline for image classification?
\newblock \emph{IEEE Transactions on Image Processing}, 24\penalty0
  (12):\penalty0 5017--5032, 2015.

\bibitem[Chen and Manning(2014)]{chen2014fast}
D.~Chen and C.~D. Manning.
\newblock A fast and accurate dependency parser using neural networks.
\newblock In \emph{Proceedings of the 2014 Conference on Empirical Methods in
  Natural Language Processing (EMNLP)}, volume~1, pages 740--750, 2014.

\bibitem[Choromanska et~al.(2014)Choromanska, Henaff, Mathieu, Arous, and
  LeCun]{choromanska2014loss}
A.~Choromanska, M.~Henaff, M.~Mathieu, G.~B. Arous, and Y.~LeCun.
\newblock The loss surface of multilayer networks.
\newblock \emph{ArXiv:1412.0233}, 2014.

\bibitem[Coates et~al.(2010)Coates, Lee, and Ng]{coates2010analysis}
A.~Coates, H.~Lee, and A.~Y. Ng.
\newblock An analysis of single-layer networks in unsupervised feature
  learning.
\newblock \emph{Ann Arbor}, 1001\penalty0 (48109):\penalty0 2, 2010.

\bibitem[Daniely et~al.(2016)Daniely, Frostig, and Singer]{daniely2016toward}
A.~Daniely, R.~Frostig, and Y.~Singer.
\newblock Toward deeper understanding of neural networks: The power of
  initialization and a dual view on expressivity.
\newblock \emph{arXiv preprint arXiv:1602.05897}, 2016.

\bibitem[Dauphin et~al.(2014)Dauphin, Pascanu, Gulcehre, Cho, Ganguli, and
  Bengio]{dauphin2014identifying}
Y.~N. Dauphin, R.~Pascanu, C.~Gulcehre, K.~Cho, S.~Ganguli, and Y.~Bengio.
\newblock Identifying and attacking the saddle point problem in
  high-dimensional non-convex optimization.
\newblock In \emph{Advances in Neural Information Processing Systems}, pages
  2933--2941, 2014.

\bibitem[Drineas and Mahoney(2005)]{drineas2005nystrom}
P.~Drineas and M.~W. Mahoney.
\newblock On the {N}ystr{\"o}m method for approximating a {G}ram matrix for
  improved kernel-based learning.
\newblock \emph{The Journal of Machine Learning Research}, 6:\penalty0
  2153--2175, 2005.

\bibitem[Duchi et~al.(2008)Duchi, Shalev-Shwartz, Singer, and
  Chandra]{duchi2008efficient}
J.~Duchi, S.~Shalev-Shwartz, Y.~Singer, and T.~Chandra.
\newblock Efficient projections onto the $\ell_1$-ball for learning in high
  dimensions.
\newblock In \emph{Proceedings of the 25th International Conference on Machine
  Learning}, pages 272--279. ACM, 2008.

\bibitem[Duchi et~al.(2011)Duchi, Hazan, and Singer]{duchi2011adaptive}
J.~Duchi, E.~Hazan, and Y.~Singer.
\newblock Adaptive subgradient methods for online learning and stochastic
  optimization.
\newblock \emph{The Journal of Machine Learning Research}, 12:\penalty0
  2121--2159, 2011.

\bibitem[Eldar and Kutyniok(2012)]{eldar2012compressed}
Y.~C. Eldar and G.~Kutyniok.
\newblock \emph{Compressed sensing: theory and applications}.
\newblock Cambridge University Press, 2012.

\bibitem[Fahlman(1988)]{fahlman1988empirical}
S.~E. Fahlman.
\newblock An empirical study of learning speed in back-propagation networks.
\newblock \emph{Journal of Heuristics}, 1988.

\bibitem[Haeffele and Vidal(2015)]{haeffele2015global}
B.~D. Haeffele and R.~Vidal.
\newblock Global optimality in tensor factorization, deep learning, and beyond.
\newblock \emph{arXiv preprint arXiv:1506.07540}, 2015.

\bibitem[Hinton et~al.(2012)Hinton, Deng, Yu, Dahl, Mohamed, Jaitly, Senior,
  Vanhoucke, Nguyen, Sainath, et~al.]{hinton2012deep}
G.~Hinton, L.~Deng, D.~Yu, G.~E. Dahl, A.-r. Mohamed, N.~Jaitly, A.~Senior,
  V.~Vanhoucke, P.~Nguyen, T.~N. Sainath, et~al.
\newblock Deep neural networks for acoustic modeling in speech recognition: The
  shared views of four research groups.
\newblock \emph{Signal Processing Magazine, IEEE}, 29\penalty0 (6):\penalty0
  82--97, 2012.

\bibitem[Isa et~al.(2010)Isa, Saad, Omar, Osman, Ahmad, and
  Sakim]{isa2010suitable}
I.~Isa, Z.~Saad, S.~Omar, M.~Osman, K.~Ahmad, and H.~M. Sakim.
\newblock Suitable mlp network activation functions for breast cancer and
  thyroid disease detection.
\newblock In \emph{2010 Second International Conference on Computational
  Intelligence, Modelling and Simulation}, pages 39--44, 2010.

\bibitem[Janzamin et~al.(2015)Janzamin, Sedghi, and
  Anandkumar]{janzamin2015generalization}
M.~Janzamin, H.~Sedghi, and A.~Anandkumar.
\newblock Generalization bounds for neural networks through tensor
  factorization.
\newblock \emph{ArXiv:1506.08473}, 2015.

\bibitem[Krizhevsky and Hinton(2009)]{krizhevsky2009learning}
A.~Krizhevsky and G.~Hinton.
\newblock Learning multiple layers of features from tiny images.
\newblock \emph{Master Thesis}, 2009.

\bibitem[Krizhevsky et~al.(2012)Krizhevsky, Sutskever, and
  Hinton]{krizhevsky2012imagenet}
A.~Krizhevsky, I.~Sutskever, and G.~E. Hinton.
\newblock Imagenet classification with deep convolutional neural networks.
\newblock In \emph{Advances in Neural Information Processing Systems}, pages
  1097--1105, 2012.

\bibitem[Lawrence et~al.(1997)Lawrence, Giles, Tsoi, and
  Back]{lawrence1997face}
S.~Lawrence, C.~L. Giles, A.~C. Tsoi, and A.~D. Back.
\newblock Face recognition: A convolutional neural-network approach.
\newblock \emph{Neural Networks, IEEE Transactions on}, 8\penalty0
  (1):\penalty0 98--113, 1997.

\bibitem[Le et~al.(2013)Le, Sarl{\'o}s, and Smola]{le2013fastfood}
Q.~Le, T.~Sarl{\'o}s, and A.~Smola.
\newblock Fastfood-approximating kernel expansions in loglinear time.
\newblock In \emph{Proceedings of the International Conference on Machine
  Learning}, 2013.

\bibitem[LeCun et~al.(1998)LeCun, Bottou, Bengio, and
  Haffner]{lecun1998gradient}
Y.~LeCun, L.~Bottou, Y.~Bengio, and P.~Haffner.
\newblock Gradient-based learning applied to document recognition.
\newblock \emph{Proceedings of the IEEE}, 86\penalty0 (11):\penalty0
  2278--2324, 1998.

\bibitem[Livni et~al.(2014)Livni, Shalev-Shwartz, and
  Shamir]{livni2014computational}
R.~Livni, S.~Shalev-Shwartz, and O.~Shamir.
\newblock On the computational efficiency of training neural networks.
\newblock In \emph{Advances in Neural Information Processing Systems}, pages
  855--863, 2014.

\bibitem[Mairal et~al.(2014)Mairal, Koniusz, Harchaoui, and
  Schmid]{mairal2014convolutional}
J.~Mairal, P.~Koniusz, Z.~Harchaoui, and C.~Schmid.
\newblock Convolutional kernel networks.
\newblock In \emph{Advances in Neural Information Processing Systems}, pages
  2627--2635, 2014.

\bibitem[Minsker(2011)]{minsker2011some}
S.~Minsker.
\newblock On some extensions of {B}ernstein's inequality for self-adjoint
  operators.
\newblock \emph{arXiv preprint arXiv:1112.5448}, 2011.

\bibitem[Mnih et~al.(2015)Mnih, Kavukcuoglu, Silver, Rusu, Veness, Bellemare,
  Graves, Riedmiller, Fidjeland, Ostrovski, et~al.]{mnih2015human}
V.~Mnih, K.~Kavukcuoglu, D.~Silver, A.~A. Rusu, J.~Veness, M.~G. Bellemare,
  A.~Graves, M.~Riedmiller, A.~K. Fidjeland, G.~Ostrovski, et~al.
\newblock Human-level control through deep reinforcement learning.
\newblock \emph{Nature}, 518\penalty0 (7540):\penalty0 529--533, 2015.

\bibitem[Rahimi and Recht(2007)]{rahimi2007random}
A.~Rahimi and B.~Recht.
\newblock Random features for large-scale kernel machines.
\newblock In \emph{Advances in Neural Information Processing Systems}, pages
  1177--1184, 2007.

\bibitem[Safran and Shamir(2015)]{safran2015quality}
I.~Safran and O.~Shamir.
\newblock On the quality of the initial basin in overspecified neural networks.
\newblock \emph{arXiv preprint arXiv:1511.04210}, 2015.

\bibitem[Sedghi and Anandkumar(2014)]{sedghi2014provable}
H.~Sedghi and A.~Anandkumar.
\newblock Provable methods for training neural networks with sparse
  connectivity.
\newblock \emph{ArXiv:1412.2693}, 2014.

\bibitem[Shalev-Shwartz et~al.(2011)Shalev-Shwartz, Shamir, and
  Sridharan]{shalev2011learning}
S.~Shalev-Shwartz, O.~Shamir, and K.~Sridharan.
\newblock Learning kernel-based halfspaces with the 0-1 loss.
\newblock \emph{SIAM Journal on Computing}, 40\penalty0 (6):\penalty0
  1623--1646, 2011.

\bibitem[Silver et~al.(2016)Silver, Huang, Maddison, Guez, Sifre, Van
  Den~Driessche, Schrittwieser, Antonoglou, Panneershelvam, Lanctot,
  et~al.]{silver2016mastering}
D.~Silver, A.~Huang, C.~J. Maddison, A.~Guez, L.~Sifre, G.~Van Den~Driessche,
  J.~Schrittwieser, I.~Antonoglou, V.~Panneershelvam, M.~Lanctot, et~al.
\newblock Mastering the game of {G}o with deep neural networks and tree search.
\newblock \emph{Nature}, 529\penalty0 (7587):\penalty0 484--489, 2016.

\bibitem[Sohn and Lee(2012)]{sohn2012learning}
K.~Sohn and H.~Lee.
\newblock Learning invariant representations with local transformations.
\newblock In \emph{Proceedings of the 29th International Conference on Machine
  Learning (ICML-12)}, pages 1311--1318, 2012.

\bibitem[Sopena et~al.(1999)Sopena, Romero, and Alquezar]{sopena1999neural}
J.~M. Sopena, E.~Romero, and R.~Alquezar.
\newblock Neural networks with periodic and monotonic activation functions: a
  comparative study in classification problems.
\newblock In \emph{ICANN 99}, pages 323--328, 1999.

\bibitem[Srivastava et~al.(2014)Srivastava, Hinton, Krizhevsky, Sutskever, and
  Salakhutdinov]{srivastava2014dropout}
N.~Srivastava, G.~Hinton, A.~Krizhevsky, I.~Sutskever, and R.~Salakhutdinov.
\newblock Dropout: A simple way to prevent neural networks from overfitting.
\newblock \emph{The Journal of Machine Learning Research}, 15\penalty0
  (1):\penalty0 1929--1958, 2014.

\bibitem[Steinwart and Christmann(2008)]{SteChr08}
I.~Steinwart and A.~Christmann.
\newblock \emph{Support vector machines}.
\newblock Springer, New York, 2008.

\bibitem[Tang(2013)]{tang2013deep}
Y.~Tang.
\newblock Deep learning using linear support vector machines.
\newblock \emph{arXiv preprint arXiv:1306.0239}, 2013.

\bibitem[VariationsMNIST()]{WinNT}
VariationsMNIST.
\newblock Variations on the {MNIST} digits.
\newblock
  \url{http://www.iro.umontreal.ca/~lisa/twiki/bin/view.cgi/Public/MnistVariations},
  2007.

\bibitem[Vincent et~al.(2010)Vincent, Larochelle, Lajoie, Bengio, and
  Manzagol]{vincent2010stacked}
P.~Vincent, H.~Larochelle, I.~Lajoie, Y.~Bengio, and P.-A. Manzagol.
\newblock Stacked denoising autoencoders: Learning useful representations in a
  deep network with a local denoising criterion.
\newblock \emph{The Journal of Machine Learning Research}, 11:\penalty0
  3371--3408, 2010.

\bibitem[Wang et~al.(2012)Wang, Wu, Coates, and Ng]{wang2012end}
T.~Wang, D.~J. Wu, A.~Coates, and A.~Y. Ng.
\newblock End-to-end text recognition with convolutional neural networks.
\newblock In \emph{Pattern Recognition (ICPR), 2012 21st International
  Conference on}, pages 3304--3308. IEEE, 2012.

\bibitem[Xiao and Zhang(2014)]{xiao2014proximal}
L.~Xiao and T.~Zhang.
\newblock A proximal stochastic gradient method with progressive variance
  reduction.
\newblock \emph{SIAM Journal on Optimization}, 24\penalty0 (4):\penalty0
  2057--2075, 2014.

\bibitem[Zhang et~al.(2015)Zhang, Lee, Wainwright, and
  Jordan]{zhang2015learning}
Y.~Zhang, J.~D. Lee, M.~J. Wainwright, and M.~I. Jordan.
\newblock Learning halfspaces and neural networks with random initialization.
\newblock \emph{arXiv preprint arXiv:1511.07948}, 2015.

\bibitem[Zhang et~al.(2016)Zhang, Lee, and Jordan]{zhang2015ell_1}
Y.~Zhang, J.~D. Lee, and M.~I. Jordan.
\newblock $\ell_1$-regularized neural networks are improperly learnable in
  polynomial time.
\newblock In \emph{Proceedings on the 33rd International Conference on Machine
  Learning}, 2016.

\end{thebibliography}


\end{document}